\documentclass[10pt, journal, twocolumn]{IEEEtran}

\usepackage[utf8]{inputenc}
\usepackage{amsmath}
\usepackage{amsthm}
\usepackage{amsfonts}
\usepackage{amssymb}
\usepackage{graphicx}
\usepackage{graphics}
\usepackage{float}
\usepackage{tikz}
\usetikzlibrary{shapes, snakes, patterns}
\usepackage{xcolor}

\newcommand{\spheading}[2][5em]{
	\rotatebox{90}{\parbox{#1}{\raggedright #2}}}
\usepackage[export]{adjustbox}
\usepackage{subcaption}
\DeclareCaptionLabelSeparator{periodspace}{.\quad}
\usepackage{algorithm}
\usepackage{algpseudocode}
\usepackage{epstopdf}
\newtheorem{theorem}{Theorem}[]
\newtheorem{lemma}[]{Lemma}
\newtheorem{corollary}[]{Corollary}
\newtheorem{remark}[]{Remark}
\newtheorem{Assumption}[]{Assumption}
\newtheorem{model}[]{Model}
\newtheorem{definition}[]{Definition}
\usepackage[justification=centering]{caption}
\usepackage{enumerate} 
\usepackage{cite}
\usepackage{suffix}
\usepackage{mathtools}
\usepackage{tabularx}
\usepackage{booktabs}
\usepackage{boldline,multirow}
\usepackage{diagbox}
\usepackage{hhline} 
\newcolumntype{C}{>{\centering\arraybackslash}X} 
\newcolumntype{L}{>{\centering\arraybackslash}m{2.1cm}}
\newcolumntype{Y}{>{\centering\arraybackslash}b{0.7cm}}
\newcolumntype{M}{>{\centering\arraybackslash}b{2.1cm}}
\newcolumntype{X}{>{\centering\arraybackslash}m{1cm}}
\newcolumntype{Z}{>{\centering\arraybackslash}b{1.75cm}}
\newcolumntype{Q}{>{\centering\arraybackslash}b{3.75cm}}
\newcolumntype{A}{>{\centering\arraybackslash}m{1.5cm}}
\newcolumntype{B}{>{\centering\arraybackslash}m{0.6667cm}}
\newcolumntype{?}{!{\vrule width 1pt}}
\newcolumntype{+}{!{\vrule width 2pt}}

\title{Subspace clustering without knowing the number of clusters: A parameter free approach}
\begin{document}
	\author{\IEEEauthorblockN{Vishnu Menon, Gokularam M,
			Sheetal Kalyani\\}
		\IEEEauthorblockA{
			Department of Electrical Engineering, Indian Institute of Technology Madras\\
			Chennai, India - 600036\\
			Email: ee16s301@ee.iitm.ac.in, ee17d400@smail.iitm.ac.in, 
			skalyani@ee.iitm.ac.in,
	}}
	\maketitle
	\begin{abstract} 
		Subspace clustering, the task of clustering high dimensional data when the data points come from a union of subspaces, is one of the fundamental tasks in unsupervised machine learning. Most of the existing algorithms for this task require prior knowledge of the number of clusters along with few additional parameters which need to be set or tuned apriori according to the type of data to be clustered. In this work, a parameter free method for subspace clustering is proposed, where the data points are clustered on the basis of the difference in the statistical distributions of the angles subtended by the data points within a subspace and those by points belonging to different subspaces. Given an initial fine clustering, the proposed algorithm merges the clusters until a final clustering is obtained. This, unlike many existing methods, does not require the number of clusters apriori. Also, the proposed algorithm does not involve the use of an unknown parameter or tuning for one. 
		A parameter free method for producing a fine initial clustering is also discussed, making the whole process of subspace clustering parameter free. The comparison of the proposed algorithm's performance with that of the existing state-of-the-art techniques in synthetic and real data sets shows the significance of the proposed method.
	\end{abstract}
	\section{Introduction}\label{sIntroduction}
	\textit{Data Clustering} is the problem of categorizing entities in the given dataset into groups called \textit{clusters} so that the entities in the same cluster are more `similar' than those from different clusters. 
	A comprehensive study of clustering algorithms is provided in \cite{xu2005survey}. 
	Very often, the dataset comprises points from a Euclidean space, and the clustering problem reduces to finding the groups which are hidden among those vectors. 
	In most techniques, distance measures are used as a similarity metric for clustering \cite{xu2005survey}. 
	However, the conventional distance measures become unreliable in high dimensions
	. In a high dimensional space, the data points are sparsely located. It is shown in \cite{beyer1999nearest} that the distance between any two high dimensional points becomes equal as the dimension $n\to\infty$.  Thus, most clustering algorithms which perform reasonably well in lower dimensions, fail in high dimensions. 
	
	Over the years, several algorithms\cite{parsons2004subspace} were developed for clustering data of large dimensions. 
	In many practical scenarios, the high dimensional data points are not uniformly distributed throughout the space but lie approximately in low dimensional structure \cite{Cherkassky1998Learning}
	. For example, the images of a face under different lighting conditions approximately lie in 9-dimensional subspace, even though they have a very large number of pixels \cite{basri2003lambertian}. 
	Principal component analysis (PCA) \cite{jolliffe2002principal} is a popular technique to retrieve a low dimensional linear subspace in which the high dimensional data points are concentrated. However, when there are multiple categories in the dataset, it is not appropriate to assume that the points lie in a single low dimensional subspace. For instance, if we have images of several faces under varying illumination conditions, then data will be lying in a union of multiple 9-dimensional subspaces
	. \textit{Subspace Clustering} addresses this problem
	by grouping data points such that each group shall contain points from a single subspace of a lower dimension \cite{vidal2011subspace}. 
	
	Subspace clustering is used extensively for 
	image representation and compression \cite{hong2006multiscale} 
	and computer vision problems like motion segmentation \cite{vidal2008multiframe}
	, face clustering \cite{ho2003clustering}, image segmentation \cite{yang2008unsupervised} and video segmentation \cite{vidal2005generalized}. It also finds applications in other fields, including hybrid system identification \cite{vidal2003algebraic}
	, gene expression analysis \cite{jiang2004cluster}, metabolic screening of new-borns \cite{achtert2006finding}, recommendation systems\cite{agarwal2005research} and web text mining \cite{zhou2014text}. 
	%
	Subspace clustering algorithms can be classified into four main types \cite{vidal2011subspace}: 
	(i) algebraic
	, (ii) iterative
	, (iii) statistical
	, (iv) spectral clustering-based. 
	{Algebraic techniques (like Generalized PCA \cite{vidal2005generalized}) assume that the data is noise-free and lie perfectly in the union of subspaces\cite{vidal2011subspace}. Sometimes they can be extended to handle moderate amounts of noise\cite{wu2001multibody}. Iterative methods (like Median K-Flats (MKF) \cite{zhang2009median}) alternate between assigning points to subspaces and recovering subspaces from each cluster. Statistical methods (like Agglomerative Lossy Compression (ALC) \cite{ma2007segmentation}) make assumptions about the generative model for the data.}
	
	Spectral clustering-based techniques have gathered a lot of attention in recent years. These methods take a two-stage approach: finding the  `affinity matrix' and then performing spectral clustering 
	\cite{von2007tutorial} on it. Each entry in the affinity matrix (sometimes referred to as graph) denotes similarity between the corresponding pair of points. The difference between different spectral clustering-based techniques is how the affinity matrix is obtained. 
	In recent years, affinity matrix is obtained using the `self-representation' of each data point with respect to all the other data points {\cite{elhamifar2013sparse,liu2012robust}}. If the data points $\textbf{m}_i$ are arranged as columns of the matrix $\textbf{M}$, then the self-representation is given by $\textbf{M}=\textbf{MZ} \ \mbox{such that} \ \textbf{Z}_{ii}=0$. After obtaining such $\textbf{Z}$, $abs(\textbf{Z})+abs(\textbf{Z}^T)$ is used as the affinity matrix (where $abs(\cdot)$ takes the absolute value of each entry in the matrix). Several techniques have been developed based on this idea. Sparse self-representation enforces the columns of $\mathbf{Z}$ to be sparse. $\ell_1$-minimization (as in Sparse Subspace Clustering (SSC) \cite{elhamifar2013sparse}
	) or Orthogonal Matching Pursuit (as in SSC-OMP \cite{dyer2013greedy}\cite{you2016scalable}) can be used to obtain such sparse representation. 
	{Least square regression (LSR)\cite{lu2012robust}} uses least-squares representations. 
	Elastic net Subspace Clustering (EnSC) \cite{you2016oracle} provides a mixture of $\ell_1$ and $\ell_2$ regularizations to obtain self-representations. 
	Few other techniques utilize low-rank self-representation like Low-Rank Recovery (LRR) \cite{liu2012robust} 
	and Low-Rank Subspace Clustering (LRSC) \cite{favaro2011closed}. Low-Rank Sparse Subspace Clustering (LRSSC) \cite{wang2013provable,wang2019provable} imposes low-rank constraint as well as sparsity constraint on the self-representation matrix. Another work \cite{lu2019subspace} uses block diagonal self-representation (BDR) 
	and performs better than several existing approaches. 
	
	There also exist several agglomerative hierarchical algorithms\cite{ma2007segmentation,leonardis2002multiple,fan2005multibody} for subspace clustering. Agglomerative (or bottom-up) hierarchical methods start with a large number of fine clusters and merge them progressively until a stopping criterion is reached. Agglomerative Lossy Compression (ALC)\cite{ma2007segmentation}, which is also a statistical subspace clustering method, finds the clustering that minimizes coding length needed to fit the data points with a Gaussian mixture. 
	A recent work, \cite{rahmani2017innovation} provides a new approach called Innovation Pursuit, which is an iterative method but can be integrated with spectral clustering to provide a new class of spectral clustering-based techniques. Currently, neural network-based clustering approaches are gaining popularity \cite{min2018survey}. 
	Especially, auto-encoder architecture \cite{song2013auto} is used to obtain sparse \cite{ji2017deep} and low rank \cite{chen2018subspace} representation for subspace clustering. These techniques can recover non-linear low dimensional structures underlying the data.
	
	
	Recently, the distributions of angles between data points \cite{cai2013distributions} have been used in \cite{menon2019structured} to develop a parameter free technique for outlier detection in high dimensions. {While some previous works in subspace clustering \cite{heckel2015robust,rahmani2017coherence,gitlin2018improving,lipor2017subspace} rely on the statistical distribution of angles, they do so only through the use of mean and all these works involve use of prior knowledge of the number of clusters (like in \cite{gitlin2018improving}) and/or involve the prior setting of one or more parameters (like in \cite{lipor2017subspace}). In contrast, the proposed work utilizes the entire distribution of angles, providing improved performance while also avoiding the need to tune any parameters.} 
	The proposed algorithm exploits %
	the difference in the statistical distributions of angles subtended by the points within a subspace and of angles subtended by the points from different subspaces 
	for achieving parameter free clustering.
	
	\subsection{Motivation}
	Many clustering algorithms require the user to supply the number of clusters to be formed beforehand. 
	In many situations, fixing the number of clusters apriori is not a good choice, especially when the knowledge about the dataset is limited. For example, in gene expression datasets, the number of clusters to be prefixed is not so clear\cite{dopazo2001methods}. There are some clustering algorithms which require one or several parameters, if not the number of clusters. Self-representation based techniques require regularization parameters \cite{elhamifar2013sparse} to be set along with the number the subspaces. Even neural network-based clustering methods require the setting of several hyperparameters. 
	There are also methods which tune for unknown parameters in the model - for instance, $\lambda$-means clustering \cite{comiter2016lambda} tunes for the parameter $\lambda$ in DP-means algorithm\cite{kulis2011revisiting}, {a general clustering algorithm used to cluster data generated by Dirichlet Process}. {Similarly, ALC \cite{ma2007segmentation} doesn't need the number of clusters apriori but requires the user to provide the distortion level $\epsilon$. Different $\epsilon$ will result in a different number of clusters in the output clustering, and hence we need to tune for $\epsilon$ for the data in hand}. When an algorithm requires one or more free parameters, the user has to set them using either cross-validation or prior knowledge about the dataset. However, parameter tuning\cite{claesen2015hyperparameter} is a difficult task, and any incorrect tuning of parameters would result in huge performance degradation.
	

	{There are several techniques in the literature to determine the number of clusters for conventional distance-based clustering of low-dimensional data\cite{pelleg2000x,tibshirani2001estimating,salvador2004determining,gupta2018fast}.  Several approaches were proposed for estimating the number of subspaces, and these estimates can then be used as an input for the subspace clustering algorithm. In \cite{liu2012robust}, the authors suggested a way to obtain the number of subspaces by soft thresholding the singular values of the Laplacian matrix of the affinity matrix. But this approach needs to set a threshold $\tau$. In \cite{heckel2015robust}, it is suggested to estimate the number of subspaces by looking for the maximum separation between successive eigenvalues of the Laplacian matrix. Though this eigen-gap heuristic approach doesn't involve any threshold, the technique is still dependent on the parameters which were used to obtain the affinity matrix. Also, this heuristic can fail when data points are noisy, and subspaces are closer\cite{von2007tutorial}.}
	
	Recently, parameter free approaches have been developed in the areas of high dimensional outlier detection \cite{menon2019structured}, sparse signal recovery \cite{DBLP:conf/icml/KallummilK18}\cite{kallummil2018high}, robust regression \cite{kallummil2018noise}, and these were shown to have results comparable with those which use the explicit knowledge about the parameters. 
	Hence, we look for a parameter free method for subspace clustering.
	\subsection{Contributions}
	Given the high-dimensional data points coming from the union of several low-dimensional subspaces, we propose an algorithm to achieve a clustering without the knowledge of the true number of clusters, $L$. This essentially consists of two steps. First, we start with an initial clustering with a large number of clusters such that each cluster {is likely to} contain points from one subspace. In the second stage, the clusters are merged to arrive at {a final clustering}. Given an initial clustering, we propose a method based on the statistical distance between the distributions of the angles subtended by the data points to find the {final} clustering without having to prefix the number of clusters. {This makes the proposed algorithm an agglomerative hierarchical method}. We also suggest a parameter free approach for initial clustering. 	
	
	The performance of the proposed algorithm is compared with state-of-the-art subspace clustering algorithms, namely, SSC-ADMM\cite{elhamifar2013sparse}, SSC-OMP\cite{you2016scalable}, EnSC\cite{you2016oracle}, TSC\cite{heckel2015robust}, ALC\cite{ma2007segmentation} and LRR\cite{liu2010robust} 
	and another recent algorithm namely, BDR-Z\cite{lu2019subspace}. We compare the algorithms in terms of Clustering Error (CE) and Normalized Mutual Information (NMI) on  synthetic as well as real datasets like Gene Expression Cancer RNA-Seq\cite{weinstein2013cancer}\cite{Dua:2019}, Novartis multi-tissue\cite{novartis}, Wireless Indoor Localization\cite{rohra2017user}\cite{Dua:2019}, Phoneme\cite{hastie1995penalized}, MNIST\cite{lecun1998gradient}, Extended Yale-B\cite{GeBeKr01,KCLee05} and Hopkins-155\cite{tron2007benchmark}. It is observed that the proposed algorithm performs on par with other methods even without the need for the number of clusters or any other parameters and has lower computational complexity.
	\subsection{Organization of the paper}
	The rest of the paper is organized as follows. In Section \ref{sProblem_and_essential_definitions}, we set up the problem and provide the definitions and notations used in this paper, along with a brief overview of the algorithm. The proposed parameter free algorithm for subspace clustering is introduced in Section \ref{sAlgorithm}. In Section \ref{sTheoretical_Analysis}, we provide the 
	analysis of our algorithm under certain assumptions on the data model. Section \ref{sNumerical_validations} provides numerical results on synthetic and real datasets and compares the performance of our algorithm with other existing algorithms. {In Section \ref{valid_results}, we discuss 
	the utility of the proposed algorithm and its pros and cons in light of all the conducted experiments.}
	
	\section{Overview and Essential Definitions}\label{sProblem_and_essential_definitions}
	The problem that we are addressing in this work is to find the clustering of a dataset comprising of high dimensional points coming from a union of subspaces. Suppose we have $N$ data points 
	$\mathbf{m}_i \in \mathbb{R}^n, \forall i \in\{ 1,2,\ldots,N\}$ and each $\mathbf{m}_i \in \mathcal{U}_1\cup\mathcal{U}_2\cup \ldots \cup\mathcal{U}_L$, where $\mathcal{U}_k$'s, $k \in \{1,2,\ldots ,L\}$, $L \ll N$ are subspaces in $\mathbb{R}^n$ with dimensions $r_k$'s respectively. We assume that there are $N_k$ points from the subspace $\mathcal{U}_k$. 
	\begin{definition}\label{dclus}
		A clustering of the dataset with $K \geq 1$ clusters is defined as $\mathcal{C}_K = \{I_1,I_2,\ldots,I_K\} $, where $I_j$'s are mutually disjoint index sets such that $\forall j = 1,2,\ldots ,K$, $I_j \subseteq \{1,2,\ldots,N\}$ and $I_j \neq \O$ with $\bigcup\limits_{j = 1}^{K} I_j = \{1,2,\ldots,N\}$. We will call $I_j$'s as constituent clusters. If $i\in I_j$, we say that $j$ is the cluster label of the point $\mathbf{m}_i$.
	\end{definition}
	\begin{definition}\label{dtrue}
		The true clustering of the dataset is defined as the clustering $\mathcal{C}_L^* = \{I_1,I_2,\ldots,I_L\}$, where $\forall j = 1,2, \ldots ,L, \  I_j  = \{i \text{	}| \text{	}\mathbf{m}_i \in \mathcal{U}_k \text{ for some	} k \in 1,2,\ldots, L\}$ and $ |I_j| = N_k$, where $|\cdot|$ denotes the cardinality of the set. i.e., each constituent cluster contains indices of all the points from a subspace and only the points from that subspace.
	\end{definition}
	
	Here, we will be working with angles subtended by high dimensional data points and their distributions. We will be using the normalized data points $\mathbf{x}_i = \frac{\mathbf{m}_i}{\|\mathbf{m}_i\|_2}$, where $\|\cdot\|_2$ denotes the $\ell_2$ norm. These points $\mathbf{x}_i \in \mathbb{R}^n, \forall i \in\{ 1,2,\ldots,N\}$ will lie in the high dimensional hypersphere $\mathbb{S}^{n-1}$. 
	Let $\theta_{ij}$ denote the angle between two data points $\mathbf{m}_i$ and $\mathbf{m}_j$, i.e., 
	\begin{equation}\label{etheta}
	\theta_{ij} = \cos^{-1}(\mathbf{x}_i^T\mathbf{x}_j) \hskip50 pt \theta_{ij} \in [0,\pi].
	\end{equation}
	  
	In this work, we hypothesize that the angles subtended by the points within the subspace and the angles subtended by the points between subspaces come from different statistical distributions and these distributions can be well approximated in many cases by distinct Gaussians with a different location and scale parameters. This is true for a random subspace model\footnote{This will be introduced in Section \ref{sAlgorithm}} and also holds for many real datasets. The proposed algorithm looks at the statistical distances between the distributions of within-cluster and between-cluster angles, where the distributions are estimated through the available angle samples. The next few definitions correspond to these samples and the related estimates, with $S^{(j)}$, $j=1,2,\ldots, |S|$ denoting the $j^{th}$ element in a set $S$. 
	\begin{definition}
		The within-cluster angle set for constituent cluster $I_k$ is defined as
		\begin{equation}\label{ewithin}
		W_k = \{\theta_{ij} \ | \ i,j \in I_k,\ i < j\}.
		\end{equation}
	\end{definition}
	If $|I_k| = t_k$, then $|W_k| = {t_k \choose 2} $ is the number of unique angles in the set.
	\begin{definition}[Within-cluster estimates]
		Given a within-cluster angle set $W_k$ for $I_k$, suppose $W^t_k \subseteq W_k$ with $|W^t_k| = t$, then the estimates corresponding to this subset are given as
		\begin{equation}\label{ewithinclus}
		\begin{aligned}
		\mu_{w_{k}t} &= \frac{1}{t} \sum\limits_{j}W^{t,{(j)}}_k\qquad\text{and}\\
		\sigma_{w_{k}t}^{2}  &= \frac{1}{t -1} \sum\limits_{j} \left(W^{t,{(j)}}_k - \mu_{w_{k}t}\right)^2.
		\end{aligned}
		\end{equation}
	\end{definition}
	Here, $\mu_{w_{k}t}$ and $\sigma_{w_{k}t}^{2}$ are respectively sample mean and sample variance of elements of the set $W^t_k$ and $W^{t,{(j)}}_k$ is the $j^{th}$ element of the set $W^{t}_k$.
	\begin{definition}
		The between-cluster angle set for two constituent clusters $I_k$ and $I_l$ is defined as
		\begin{equation}\label{ebw}
		B_{kl} = \{\theta_{ij} \text{	}| \text{	} i \in I_k, \ j \in I_{l}\}.
		\end{equation}
	\end{definition}
	Clearly, $|B_{kl}| = t_kt_l$, the number of possible cross angles. 
	\begin{definition}[Between-cluster estimates]
		Given a between-cluster angle set $B_{kl}$ for $I_k$ and $I_l$, suppose $B^t_{kl} \subseteq B_{kl}$ with $|B^t_{kl}| = t$, then the estimates corresponding to this subset are given as
		\begin{equation}\label{ebwclus}
		\begin{aligned}
		\mu_{b_{kl}t} &= \frac{1}{t} \sum\limits_{j }B^{t,{(j)}}_{kl}\qquad\text{and}\\
		\sigma_{b_{kl}t}^{2}  &= \frac{1}{t -1} \sum\limits_{j } \left(B^{t,{(j)}}_{kl} -\mu_{b_{kl}t} \right)^2.
		\end{aligned}
		\end{equation}
	\end{definition}
	Here, $\mu_{b_{kl}t}$ and $\sigma_{b_{kl}t}^{2}$ are respectively sample mean and sample variance of elements of the set $B^t_{kl}$ and $B^{t,{(j)}}_{kl}$ is the $j^{th}$ element of the set $B^t_{kl}$.
	
	Bhattacharyya distance is a very popular measure used for measuring distances between probability distributions \cite{bhattacharyya1943measure}. We use its empirical version \cite{4309450} as our key metric.
	\begin{definition}\label{ddist}
		Given two constituent clusters $I_k$ and $I_l$, the distance $D_{kl}$ between them is defined as the Bhattacharyya distance \cite{bhattacharyya1943measure} between the distribution of angles in $W_k$ and $B_{kl}$, i.e., $D_{kl} = D_B(\theta_{W_k}, \theta_{B_{kl}})$, where $\theta_{W_k}$ is the statistical distribution of angles in $W_k$ and $\theta_{B_{kl}}$ is the statistical distribution of angles in $B_{kl}$ and $D_B()$ is the Bhattacharyya distance between them. The empirical version used here is defined as
		\begin{equation}\label{estatalter}
		d_{kl} \! =\! \dfrac{1}{4} \!\left[\dfrac{	(\mu_{w_{k}t} - \mu_{b_{kl}t} )^2}{\sigma^2_{w_{k}t}+\sigma^2_{b_{kl}t}}\!+\! \log_{e}\!\left(\!\frac{1}{4}\!\left[\!\dfrac{\sigma^2_{w_{k}t}}{\sigma^2_{b_{kl}t}} \!+\! \dfrac{\sigma^2_{b_{kl}t}}{\sigma^2_{w_{k}t}}\!\right]\!\!+\!\frac{1}{2}\!\right)\!\right]\!\!.
		\end{equation}
	\end{definition}
	In the algorithm, We start from a fine clustering and merge those clusters which are closest in terms of the empirical Bhattacharyya distance until we reach a final clustering. The closeness is measured in terms of the scores, as defined below. 
	\begin{definition}
		Score of a constituent cluster $I_j$ in a clustering $\mathcal{C}_K$ is given by
		\begin{equation}\label{escrI}
		\eta_j = \underset{l = 1,2,\ldots, K , \ l\neq j}{\min} \ d_{jl}.
		\end{equation}
		
		Also, we define the partner of a 
		cluster as the one which produces its score. i.e., if $j' = \underset{l = 1,2,\ldots, K, \ l\neq j}{\arg \min}\text{	}d_{jl}$, then $I_{j'}$ is the partner of $I_j$.
	\end{definition} 
	\begin{definition}
		Score of a clustering $\mathcal{C}_K$ is given by
		\begin{equation}\label{escrclus}
		\gamma_K =\underset{i  = 1,2, \ldots ,K}{\min}\text{	} \eta_i .
		\end{equation}	
		 
		 Let $i^* =\underset{i  = 1,2, \ldots, K}{\arg\min}\text{	} \eta_i$ and also let $I_{j^*}$ be the partner of $I_{i^*}$. Then, we call $(I_{i^*}, I_{j^*})$ as a mergeable pair of $\mathcal{C}_K$.
	\end{definition}
	The algorithm is explained in detail in Sections \ref{sAlgorithm} and \ref{sTheoretical_Analysis}.\\[1ex]	
	\noindent\textbf{Other Notations:} 
	$\mathbb{P}(\cdot)$ denotes the probability measure. $\mathcal{N}(\mu, \sigma^2)$ denotes the normal distribution with mean $\mu$ and variance $\sigma^2$. Let $F_{\mathcal{N}}(\cdot)$ denote the cdf of the standard normal distribution, 
	$\Gamma(\cdot)$ denote the gamma function and $\chi^2_k$ denote the standard chi-squared distribution with $k$ degrees of freedom. { $F_{\chi^2_{k}}(\cdot)$ denotes the cdf of  $\chi^2_k$ distribution 
	and $F_{\chi^2(1,\lambda)}(\cdot)$ denotes the cdf of a non-central $\chi^2$ distribution with 1 dof and the non-centrality parameter $\lambda$}. 
	$\beta'(a,b)$ denotes beta prime distribution with parameters $a$ and $b$ and $F_{\beta'(a,b)}(\cdot)$ is its cdf. 
	Also, $w.p$ indicates `with probability'. $\lfloor x\rfloor$ is the floor of $x$ and $\lceil x\rceil$ is the ceiling of $x$. 
	$O()$ denotes the Big O notation for complexity. $abs(x)$ denotes the absolute value of $x$. 
	
	\section{Algorithm}\label{sAlgorithm}
	In this section, we will explain the proposed algorithm for clustering. 
	First, we will assume that some process gives us an initial fine clustering $\mathcal{C}_P$ and develop an algorithm for merging. Then, we will discuss a method that will select the appropriate clustering from the set of outputs after the merging process such that the final clustering estimate is close to the true clustering. We will also discuss possible initial clustering methods. {Here, we give a theoretical framework which forms the basic idea of our proposal. The exact definitions of the distances we used and the derivation of the thresholds can be found in Section \ref{sTheoretical_Analysis}.}
	\subsection{Algorithm for merging}
	Suppose we have a fine clustering $\mathcal{C}_P$, with $P\gg L$. The proposed algorithm runs through from $P$ to $2$, by merging clusters and then selects the appropriate clustering through methods described later. First, we will see the merging process starting with $K = P$.
	\begin{itemize}
		\item[1.] Let the current clustering be $\mathcal{C}_K$ with constituent clusters $I_1,I_2 ,\ldots, I_K$. For each $I_j$, calculate the distances between constituent clusters as per Definition \ref{ddist} and then for each $I_j$, find its score $\eta_j$ using (\ref{escrI}) and also get their respective partners. 
		\item[2.] Calculate the score of the clustering $\gamma_K$ as in (\ref{escrclus}) and find the mergeable pair in the current clustering $\mathcal{C}_K$. Let the mergeable pair be $(I_{k_1},I_{k_2})$.
		\item[3.] Merge clusters $I_{k_1}$ and $I_{k_2}$ and form the new clustering $\mathcal{C}_{K-1}$ with $K-1$ constituent clusters.
		\item[4.] Repeat steps $1 - 3$ until $K=2$. 
	\end{itemize}

	In this merging, we form a series of clusterings $\mathcal{C}_P,\mathcal{C}_{P-1},\ldots, \mathcal{C}_2$ where each subsequent clustering is formed by merging the mergeable pair in the previous clustering, or in other words we simply combine the closest clusters in terms of the distance between distributions of the within-cluster angles and between-cluster angles. 
	
	The intuition behind this merging process is as follows. We hypothesize that the angles between points from the same subspace come from one statistical distribution, and the angles between points coming from different subspaces follow a different distribution. A theoretical treatment of this hypothesis and the motivation behind it is provided in Section \ref{sTheoretical_Analysis}. 
	This hypothesis implies that, when there are constituent clusters with points from the same subspace in them, the statistical distance between the distributions of angles within the constituent cluster and between the constituent clusters would be close to $0$, provided we have enough angle samples in each set.
		
	To clarify, suppose we look at the clustering $\mathcal{C}_P$. Take the constituent cluster $I_1$ and suppose that $I_{k}$ also contains only points from the same subspace, then the angles in $W_1$ and $B_{1k}$ come from the same distribution. This means that we have a low value close to $0$ for $d_{1k}$, which is the measure of divergence between the distributions of angles in $W_1$ and $B_{1k}$. On the other hand, if $I_j$ is a constituent cluster with points from a different subspace to those in $I_1$, then the angles in $W_1$ and $B_{1j}$ come from different distributions, and hence the distance $d_{1j}$ will be high and bounded away from $0$. We calculate $\eta_1$ as the minimum distance made by the constituent cluster $I_1$. When the clustering contains at least another constituent cluster with only points from the same subspace as is the case above, then the partner of that constituent cluster will be one of those clusters with only points from the same subspace. i.e., suppose for the above case, let $I_{k}$, $I_{l}$ and $I_{m}$ have only points from the same subspace as in $I_1$, then the partner of $I_1$ will be either $I_{k}$, $I_{l}$ or $I_{m}$. Hence, when we look at the bigger picture of a clustering, whenever a clustering contains at least a pair of constituent clusters having only points from the same subspace, the mergeable pair shall contain only points from the same subspace and the clustering score, $\gamma_K$ will be very close to $0$. In each step, we merge the closest clusters in terms of $d_{kl}$, which ensures that points from the same subspace get clustered together as we keep merging. The merging process is summarized in Algorithm \ref{alg_merge}.
	\begin{algorithm}[h]
		\caption{Procedure for Merging} \label{alg_merge}
		\textbf{Input}: Initial clustering $\mathcal{C}_P$, normalized data matrix $\mathbf{X}$.\\
		\textbf{Initial calculation}: Calculate $\theta_{ij}\ \forall i,j = 1,2,\ldots, N$ as in (\ref{etheta}).\\
		\textbf{Initialization}: $K=P$.
		\begin{algorithmic}[1]
			\State Calculate $\gamma_K$ for the current clustering $\mathcal{C}_K$ as in (\ref{escrclus}).
			\State Merge the mergeable pair in $\mathcal{C}_K$ and form $\mathcal{C}_{K-1}$.
			\State Repeat steps $1 - 2$ until $K = 2$.
		\end{algorithmic}
		\textbf{Output}: Clusterings $\mathcal{C}_P,\mathcal{C}_{P-1},\ldots, \mathcal{C}_2$.
	\end{algorithm}
	\subsection{Selecting optimal clustering}
	If we start with a clustering $\mathcal{C}_P$ such that each constituent cluster in $\mathcal{C}_P$ only contains indices of points from one subspace, then through the merging process described in the previous subsection, at some point, we will arrive at a clustering $\mathcal{C}_{\hat{L}}$. At this point, no two constituent clusters have points from the same subspace, which means that no cluster pair can form a distance $d_{ij}$ that is close to $0$. Hence, the cluster score $\gamma_{\hat{L}}$ will take a jump from a value close to zero to a higher magnitude. This is what we try to exploit in our algorithm to find $\hat{L}$.
	
	Also note that here we compute the statistical distances using the angle samples we have in the within-cluster and between-cluster sets, i.e., for constituent clusters $I_i$ and $I_j$, $\theta_{W_i}$ and $\theta_{B_{ij}}$ are estimated distributions and $d_{ij}$ is the empirical Bhattacharyya distance. Suppose we have $t$ angle samples, we can state the following on the behaviour of $\gamma_K$:
	\begin{itemize}
		\item[a)] Suppose for a clustering $\mathcal{C}_K$, $I_i$ and $I_j$ contain only points from the same subspace $\mathcal{U}_a$ such that the angles between the points in $\mathcal{U}_a$ are distributed with a distribution $p_{\mathcal{U}_a}$, then as the number of angle samples $t \to \infty$,  $\theta_{W_i} \to p_{\mathcal{U}_a}$ and $\theta_{B_{ij}} \to p_{\mathcal{U}_a}$ $\Rightarrow$ $d_{ij} \to 0$. In other words, $d_{ij}$ will be very close to $0$ if one has a large number of angle samples for estimating the distribution. 
		\item[b)] Hence in $\mathcal{C}_K$, for $I_i$, if $\exists I_j$ as described in a), then $\eta_i \to 0$ given a large number of angle samples. This is because of the definition of $\eta_i$, which is the minimum distance that a constituent cluster makes. 
		\item[c)] From the above, for a clustering $\mathcal{C}_K$, if there exists at least one such $I_i, I_j$ pair as a), then $\gamma_K  = \underset{l  = 1,2, \ldots, K}{\min}\text{	} \eta_l \to 0$ at large enough number of angle samples.
		\item[d)] Suppose that such a pair as a) does not exist in $\mathcal{C}_K$, i.e., for any $i$ and $j$, $I_{i}$ and $I_{j}$ contain points from different subspaces, say $\mathcal{U}_a$ and $\mathcal{U}_b$ respectively. Then, as $t \to \infty$,  $\theta_{W_i} \to p_{\mathcal{U}_a}$ and $\theta_{B_{ij}} \to q$, where $q$ is the distribution of angles between points from different subspaces. Then, as $t$ increases, $d_{ij} \to D_B(p_{\mathcal{U}_a}, q)$ and since these are different distributions, $d_{ij}$ is bounded away from $0$ $\forall i,j$, which in turn leads to $ \gamma_K$ being bounded away from $0$.
		\item[e)] Also note that, whenever $I_{i}$ and $I_{j}$ contains a mixture of points from different subspaces, we cannot make an assertion on the nature of distributions in each set and hence the distance metric $d_{ij}$ becomes unpredictable.
	\end{itemize}

	Hence while merging, suppose we arrive at a true clustering at $K = \hat{L}$, then there exists no more mergeable pair which have points from the same subspace. So, we can state the following assuming that $t \to \infty$.
	\begin{itemize}
		\item For $K > \hat{L}$, $\gamma_K \to 0$.
		\item For $K = \hat{L}$, $\gamma_K$ is bounded away from zero.
		\item For $K < \hat{L}$, $\gamma_K$ behaviour is unknown.
	\end{itemize}

	Now, we will describe the 
	method that can identify $\hat{L}$ from the calculated $\gamma_K$'s.
	This method is essentially a thresholding scheme, which uses a threshold $\zeta_{K}$ on the scores $\gamma_K$ and looks for the first crossing of this threshold as our {final} clustering. The algorithm is summarized in Algorithm \ref{alg_zeta}. 
	\begin{algorithm}[h]
		\caption{Thresholding with $\zeta_{K}$} \label{alg_zeta}
		\textbf{Input}: $\mathcal{C}_P,\mathcal{C}_{P-1},\ldots,\mathcal{C}_2$ and associated $\gamma_K$'s.
		\begin{algorithmic}[1]
			\State For each $K$, calculate $\zeta_{K}$ as in (\ref{ethr}).
			\State $\hat{L} = \max \{K \text{	}| \text{	} \gamma_K > \zeta_{K}\}$.\label{eq_hat_L}
			\State $\hat{\mathcal{C}}_{\hat{L}} = \mathcal{C}_{\hat{L}}$.
		\end{algorithmic}
		\textbf{Output}: Estimated optimal clustering $\hat{\mathcal{C}}_{\hat{L}}$.
	\end{algorithm}
	
	In Section \ref{sTheoretical_Analysis}, we formulate a threshold $\zeta_{K}$ theoretically, which with a very high probability ensures that the scores $\gamma_K \leq \zeta_{K}$ whenever the mergeable pair in $\mathcal{C}_K$ contains only points from the same subspace, i.e., through this threshold we ensure that whenever $K > \hat{L}$, 
	the scores $\gamma_K \leq \zeta_{K}$. So, $\hat{L}$ is the first instance where the score exceeds the threshold, and we find that as in step \ref{eq_hat_L} in Algorithm \ref{alg_zeta}. The formulation of the high probability threshold involves a statistical analysis on the distances and is discussed in Section \ref{stheory}, where we derive the threshold $\zeta_{K}$. For this derivation, we use a data model described in Section \ref{sassuumpttions}, where we also provide the motivation behind the model.
	\subsection{Initial Clustering}\label{sinit_clust}
	The algorithm discussed previously requires as input an initial fine clustering $\mathcal{C}_P$ with $P \gg L$. Here, we discuss possible methods to get this clustering from data. One way is to use well-known algorithms for clustering, setting them for over-estimating the number of clusters. For instance, one could use K-Means clustering by setting a large value for K. Here we will use a method based on the closeness of points while keeping a minimum of 3 points per cluster, which is the only assumption made. The following steps describe the algorithm for initial clustering.
	\begin{itemize}
		\item[1.] For each data point, find the two closest points in terms of the acute angles between them, i.e. $\cos^{-1}(abs(\mathbf{x}_i^T\mathbf{x}_j))$. Let us call them allies of a point.
		\item[2.] Starting from any point chosen randomly, form clusters with the point and it's two allies. Run through the points forming such clusters, avoiding repetition of points in clusters. Here, a new cluster is formed only when a point and both its allies are not already allocated to a cluster. Hence, in this run, a lot of points go unallocated. 
		\item[3.] For all the points left out, add them to the cluster of its closest ally, if it has a cluster allocated or else add them to the cluster of its second ally. After this run, all points are added to some cluster. 
	\end{itemize}

	This forms our initial clustering, with at least 3 points in each constituent cluster. Please note that this initialization scheme does not guarantee that the initial clustering has constituent clusters with points only from the same subspace. {This method of initial clustering is similar to \cite{park2014greedy} and could be improved upon as part of future work.} 
	\section{Theoretical Analysis}\label{sTheoretical_Analysis}
	To provide theoretical results, we require a model for data points from different subspaces. In this section, we will describe the assumptions that we make on the data model and the motivation behind this assumption. Then, under this model, we analyse the algorithm theoretically and derive the threshold $\zeta_{K}$. {Here we show theoretically that, under the Gaussian assumption on the nature of distributions of angles, the score $\gamma_K \leq \zeta_K$ with high probability, whenever a mergeable pair exists in the current clustering. This will result in the algorithm achieving a true clustering with high probability.}
	\subsection{Assumptions used and their motivations}\label{sassuumpttions}
	The distribution of angle between two high dimensional points is studied in detail in \cite{cai2013distributions}. When independent points are chosen uniformly at random from $\mathbb{S}^{n-1}$, the distribution of the angles  between any two of them is approximately Gaussian with mean $\frac{\pi}{2}$ and variance $\frac{1}{n-2}$ \cite{cai2013distributions}. Even if all the points are independent, the angles which involve the same point are only pairwise independent \cite{cai2012phase}. Let us look at the following model:
	\begin{model}\label{model_ss}
		The subspaces $\mathcal{U}_i$'s, $i = 1,2,\ldots,L$ are chosen uniformly at random from the set of all $r_i$ dimensional subspaces respectively and the normalized points in each subspace are sampled uniformly at random from the 
		$\mathcal{U}_i\cap\mathbb{S}^{n-1}$. 
	\end{model}
	{Note that Model \ref{model_ss} is the fully-random model as used in \cite{heckel2015robust,soltanolkotabi2012geometric}.} In this model, we can use results proved in \cite{cai2013distributions,menon2019structured} to state the following:
	\begin{lemma}\label{lemmaold}
		In Model \ref{model_ss}, let $\mathcal{C}_L^*$ be a true clustering.
		\begin{itemize}
			\item[a)] When $i, j \in I_k$, with $I_k$ corresponding to the subspace $\mathcal{U}_a$,  $\theta_{ij}$'s are identically distributed with an expected value of $\frac{\pi}{2}$ and its pdf is given by
			$h_{r_a}(\theta) = \dfrac{1}{\sqrt{\pi}}\dfrac{\Gamma\left(\frac{r_a}{2}\right)}{\Gamma\left(\frac{r_a-1}{2}\right)} (\sin\theta)^{r_a-2}, \ \ \theta \in [0, \pi]$.
			\item[b)] When $i\in I_k$, $j \in I_l$, then $\theta_{ij}$'s are identically distributed with an expected value of $\frac{\pi}{2}$ and its pdf is given by
			$h_{n}(\theta) = \dfrac{1}{\sqrt{\pi}}\dfrac{\Gamma\left(\frac{n}{2}\right)}{\Gamma\left(\frac{n-1}{2}\right)} (\sin \theta)^{n-2}, \ \ \theta \in [0, \pi]$. {Also, $\theta_{ij}$ converges in distribution to $\mathcal{N}\left(\frac{\pi}{2},\frac{1}{n-2}\right)$ as $n \to \infty$ and the rate of convergence is of $O\left(\frac{1}{n}\right)$.}
		\end{itemize}
	\end{lemma}
	\begin{proof}
		Please refer to Appendix \ref{appendix1}.
	\end{proof}
	Also, we state the following remark.
	\begin{remark}\label{r1}
		The pdf $h_p(\theta)$ from Lemma \ref{lemmaold} can be approximated by the pdf of Gaussian distribution with mean $\frac{\pi}{2}$ and variance $\frac{1}{p-2}$, specifically for $p \geq 5$ as validated in \cite{cai2013distributions}. 
	\end{remark}
	
	\par Through Lemma \ref{lemmaold} and Remark \ref{r1}, we can see that in Model \ref{model_ss}, the angles between points coming from the same subspace and the angles between points of different subspaces both follow Gaussian distribution with different variances. 
	Model \ref{model_ss} 
	gives us a framework to work with. However, it is too restrictive. Hence, we generalize this model: We assume that the angles between points in the same subspace are approximately Gaussian distributed with some mean $\mu_w$ and variance $\sigma_w^2$ and those coming from points of different subspaces also approximately Gaussian distributed with some other mean $\mu_b$ and variance $\sigma_b^2$, with all of the parameters varying according to the data used and the pair of clusters considered. Essentially, the distributions of angles subtended by points in the same subspace and those from different subspaces have different distributions, both of which 
	can be well approximated by Gaussian distributions. 
	We can see that this assumption holds in many cases. Fig. \ref{fig_rna_angles} shows the histogram of within-cluster angles and between-clusters angles in Gene Expression Cancer RNA-Seq dataset \cite{weinstein2013cancer}\cite{Dua:2019}, which are approximately Gaussian distributed with different parameters\footnote{For this illustration, we have considered the clusters - LUAD and PRAD from the dataset.}.  This is the motivation behind making the following assumption in this work.
	\begin{center}
		\begin{figure}[h!]
			\includegraphics[width=\linewidth]{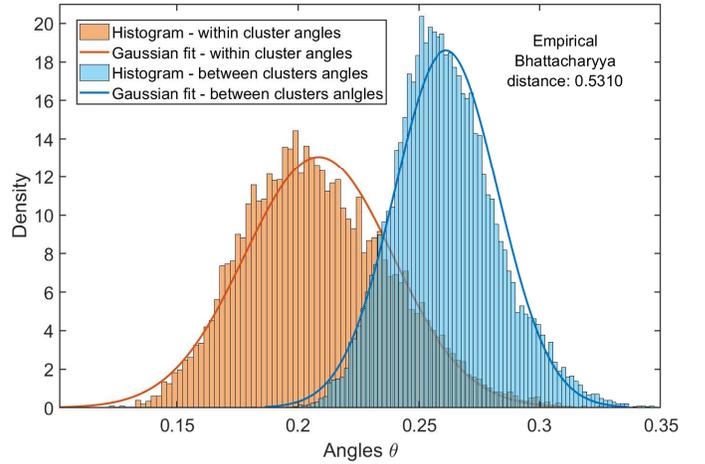}
			\caption{Distribution of within-cluster angles and between-clusters angles in Gene Expression Cancer RNA-Seq dataset}
			\label{fig_rna_angles}
		\end{figure}
	\end{center}
	\begin{Assumption}\label{amain1} 
For $\mathbf{x}_1,\mathbf{x}_2, \ldots \mathbf{x}_{N_a} \in \mathcal{U}_a$, the angles between them, $\theta_{ij}$'s are identically distributed as $ \mathcal{N}(\nu_{a},\rho^2_{a})$. {For any $\mathbf{x}_i \in \mathcal{U}_a$ and $\mathbf{x}_j \in \mathcal{U}_b$, the angle $\theta_{ij}$ is distributed as $ \mathcal{N}(\nu_{ab},\rho^2_{ab})$}. For any $i$, $j$, $k$ and $l$, the angles $\theta_{ij}$, $\theta_{ik}$ and $\theta_{il}$ are not mutually independent but are pairwise independent, i.e., $\theta_{ij}$ and $\theta_{ik}$ are independent, so are $\theta_{ij}$ and $\theta_{il}$.
	\end{Assumption} 
	\par{Please note that the model only assumes that the distributions of within-cluster angles and between-cluster angles are Gaussian and does not require the subspaces to be linearly independent.} Also, note that trivially the angles are mutually independent if they do not have a common point involved. We assume throughout this work that the angles formed by points within a subspace 
	have a Gaussian distribution %
	and those between points coming from different subspaces have a different Gaussian distribution. As seen in the previous section, the proposed algorithm for finding the {final} clustering from a given clustering with a large number of constituent clusters is based on distances between distributions of angles within and between constituent clusters.  We will model the angles using Assumption \ref{amain1} and derive the threshold $\zeta_{K}$ theoretically in the next subsection. 
	
	\subsection{Theoretical behaviour of scores}\label{stheory}
	For theoretical analysis, we will work under Assumption \ref{amain1}. Consider a clustering $\mathcal{C}_K$. Let $I_i$ and $I_j$ be two constituent clusters in $\mathcal{C}_K$. Let $|I_i| = \omega_i$ and $|I_j| = \omega_j$. Here, we will do the following:
	\begin{itemize}
		\item[i] Characterize $d_{ij}$ mathematically looking at its statistical properties when $I_i$ and $I_j$ contain only points from the same subspace $\mathcal{U}_a$ {and also when $I_i$ and $I_j$ contain points from different subspaces $\mathcal{U}_a$ and $\mathcal{U}_b$}.
		\item[ii] We will use this to develop the threshold $\zeta_{K}$, which is a high probability upper bound on $\gamma_K$ whenever there exists some $I_i,I_j \in \mathcal{C}_K$ which contains only points from the same subspace. 
	\end{itemize}
	
	When we use Assumption \ref{amain1}, $d_{ij}$ between constituent clusters $I_i$ and $I_j$ is the Bhattacharyya distance between two normal distributions. We calculate it empirically using angle samples as defined in (\ref{estatalter})
	. The properties of this are studied in detail in \cite{4309450}. 
	\par Under Assumption \ref{amain1}, we have only assumed pairwise independence of angles and not mutual independence when they involve the same data point. In this section, we will ensure that the estimates are generated by independent angles by designing the subsets $W^t_i$ and $B^t_{ij} $ as such. We want to ensure that the samples used for the estimates $\mu_{w_{i}t_{ij}}$ and $\mu_{b_{ij}t_{ij}}$ are independent samples and also ensure that $\mu_{w_{i}t_{ij}}$ and $\mu_{b_{ij}t_{ij}}$ are independent with respect to each other. For this, we have to ensure that we only pick at most two angles formed by a point in these estimates. Lemma \ref{lsubsets} in Appendix \ref{appendix2} designs such subsets $W^{t_i}_i $ and $B^{t_j}_{ij}$ with $t_i =  \lfloor \frac{\omega_i}{2} \rfloor$ and $t_j  = \min (\omega_i,\omega_j)$, where the samples in these subsets are independent. Let $t_{ij} = \min(t_i,t_j) = \min(\lfloor \frac{\omega_i}{2} \rfloor, \omega_j)$.
	
	 Further in the analysis, we will assume that we use only $t_{ij}$ samples each from the sets $W^{t_i}_i $ and $B^{t_j}_{ij}$ for getting our estimates. This helps in simplifying the analysis without compromising on its crux. Let $W^{t}_i $ and $B^{t}_{ij}$ be these sets, and $\mu_{w_{i}t_{ij}},\, \sigma^2_{w_{i}t_{ij}}$ and $\mu_{b_{ij}t_{ij}},\, \sigma^2_{b_{ij}t_{ij}}$ be the corresponding estimates. Let us also denote $X_{ij} = (\mu_{w_{i}t_{ij}} - \mu_{b_{ij}t_{ij}} )$, $Y_{ij} = \sigma^2_{w_{i}t_{ij}}+\sigma^2_{b_{ij}t_{ij}}$, $Z_{ij} = {\sigma^2_{w_{i}t_{ij}}}/{\sigma^2_{b_{ij}t_{ij}}}$ and also $U_{ij} = {X_{ij}^2}/{Y_{ij}}$ and $V_{ij} = Z_{ij} + \dfrac{1}{Z_{ij}}$. Then $d_{ij} = \frac{1}{4}\!\left(U_{ij}+ \log_{e}\left[\frac{V_{ij}}{4}+\frac{1}{2}\right]\right)$. We will first look at the distribution and properties of these estimates.
	{
	\begin{lemma}\label{ldistr}
		Under Assumption \ref{amain1}:
		\begin{itemize}
			\item [a)] If $I_i$ and $I_j$ contain points only from the same subspace $\mathcal{U}_a$, then the estimates $\mu_{w_{i}t_{ij}}$ and $\sigma^2_{w_{i}t_{ij}}$ are independent and so too are $\mu_{b_{ij}t_{ij}}$ and $\sigma^2_{b_{ij}t_{ij}}$. Also,
			\begin{align}\label{edists}
				\frac{t}{2\rho_{a}^2}X_{ij} ^2 \sim  \chi^2_{1} \, \text{		and		}\, \sigma^2_{w_{i}t_{ij}},\, \sigma^2_{b_{ij}t_{ij}}\! \sim\! \frac{\rho_a^2}{t_{ij}-1} \chi^2_{t_{ij}-1}.\!
			\end{align}
			\item[b)] If $I_i$ contain points only from subspace $\mathcal{U}_a$ and $I_j$ from $\mathcal{U}_b$:
			\begin{equation}\label{edistd}
			\begin{aligned}
			&\frac{t_{ij}}{\rho_a^2+\rho_{ab}^2}X_{ij}^2 \sim \chi^2\left(k = 1, \lambda = t_{ij}\frac{(\nu_a- \nu_{ab})^2}{\rho_a^2+\rho_{ab}^2}\right),\!\\
			&\sigma^2_{w_{i}t_{ij}}\! \sim\! \frac{\rho_a^2}{t_{ij}-1} \chi^2_{t_{ij}-1},\ \sigma^2_{b_{ij}t_{ij}} \!\sim\! \frac{\rho_{ab}^2}{t_{ij}-1} \chi^2_{t_{ij}-1}.\!
			\end{aligned}
			\end{equation}
		\end{itemize}		
	\end{lemma}
	\begin{proof}
		Please refer to Appendix \ref{appendix2}.
	\end{proof}}

	{Now, we will look at $d_{ij}$, given in (\ref{estatalter}) for both cases where the points come from the same subspace and different subspaces. 
	\begin{theorem}\label{t1}
		Suppose $I_i$ and $I_j$  contain points from the same subspace $\mathcal{U}_a$ with $t_{ij}$ independent angle samples used for estimating the sample means and variances.  Then under Assumption \ref{amain1},
		\begin{equation}\label{esamess}
			d_{ij} \leq \frac{1}{\sqrt{t_{ij}-1}} \qquad w.p \geq 1-\epsilon_{t_{ij}}. 
		\end{equation}
        Here, $ \epsilon_{t_{ij}} =  2 - F_{\beta'\left(\frac{1}{2},t_{ij}-1\right)}\!\!\left(\frac{t_{ij}}{(t_{ij}-1)^{\frac{3}{2}}}\right) - F_{\beta'\left(\frac{t_{ij}-1}{2},\frac{t_{ij}-1}{2}\right)}\!\!\left(\frac{c + \sqrt{c^2-4}}{2}\right) +  F_{\beta'\left(\frac{t_{ij}-1}{2},\frac{t_{ij}-1}{2}\right)}\!\!\left(\frac{c - \sqrt{c^2-4}}{2}\right)$, where $c = 4\Big(e^{\frac{2}{\sqrt{t_{ij}-1}}}  - \frac{1}{2}\Big)$.
	\end{theorem}
	\begin{proof}
		First, look at $Y_{ij}$, we use the result that if $A_1 \sim \chi^2_{a_1}$, $A_2 \sim \chi^2_{a_2}$, then $A_1 + A_2 \sim \chi^2_{a_1+a_2}$. Using this and (\ref{edists}) in Lemma \ref{ldistr}, $\sigma^2_{w_{i}t_{ij}}+\sigma^2_{b_{ij}t_{ij}} \sim  \frac{\rho_a^2}{t_{ij}-1} \chi^2_{2(t_{ij}-1)}	\Rightarrow \frac{t_{ij}-1}{\rho_a^2}\, Y_{ij} \sim  \chi^2_{2(t_{ij}-1)}$. Again from (\ref{edists}) in Lemma \ref{ldistr}, 
		\begin{align*}
		U_{ij}   = \frac{X_{ij} ^2}{Y_{ij}} = \frac{2(t_{ij}-1)}{t_{ij}} \frac{A_1}{A_2},
		\end{align*}
		where $A_1 \sim \chi^2_{1} $ and $A_2 \sim \chi^2_{2(t_{ij}-1)}$. Using the result that the ratio of two independent chi-squared random variables follows a beta prime distribution \cite{johnson1995continuous}, $\frac{A_1}{A_2} \sim \beta'\left(\frac{1}{2},t_{ij}-1\right)$. Hence,
		$\frac{t_{ij}}{2(t_{ij}-1)}\, U_{ij} \sim \beta'\left(\frac{1}{2},t_{ij}-1\right)$. This leads to the following: 
		\begin{equation}\label{et1bnd}
		U_{ij} \leq \frac{2}{\sqrt{t_{ij}-1}} \hskip10pt w.p	\hskip10pt  F_{\beta'\left(\frac{1}{2},t_{ij}-1\right)}\!\!\left(\frac{t_{ij}}{(t_{ij}-1)^{\frac{3}{2}}}\right)\!.
		\end{equation}
		Let $c = 4\left(e^{\frac{2}{\sqrt{t_{ij}-1}}}  - \frac{1}{2}\right)$. Now, we look at:$$\mathbb{P}\left(\log_{e}\left(\frac{V_{ij}}{4}+\frac{1}{2}\right) \leq \frac{2}{\sqrt{t_{ij}-1}}\right) \equiv \mathbb{P}\left(V_{ij} \leq c\right).$$ $$\mathbb{P}(V_{ij} \leq c) = \mathbb{P}\left(\!Z_{ij}+\dfrac{1}{Z_{ij}} \leq c\!\right) =   \mathbb{P}\left(Z_{ij}^2 -cZ_{ij} +1 \leq 0\right).$$
Consider $Z_{ij}^2 -cZ_{ij} +1\leq 0.$ The roots of this quadratic are  $(z_0,z_0')=\frac{c \pm \sqrt{c^2-4}}{2}$. Note that $e^{\frac{2}{\sqrt{t_{ij}-1}}}\geq 1 
 \Rightarrow c^2 \geq 4.$ Thus, $z_0$ and $z_0'$ are real with $z_0'\leq z_0$. Hence, $Z_{ij}^2 -cZ_{ij} +1\leq 0\Rightarrow (Z_{ij}-z_0')(Z_{ij}-z_0)\leq0\Rightarrow Z_{ij}\in[z_0',z_0]$. Thus,
		\begin{align*}
		&\mathbb{P}(Z_{ij}^2\!-\!cZ_{ij}\!+1\!\leq\!0)=\mathbb{P}\bigg(\frac{c\!-\!\sqrt{c^2\!-\!4}}{2}\!\leq Z_{ij} \leq\! \frac{c\!+\! \sqrt{c^2\!-\!4}}{2}\bigg).
		\end{align*}
		Note that $Z_{ij} = {\sigma^2_{w_{i}t_{ij}}}/{\sigma^2_{b_{ij}t_{ij}}} = {A_1}/{A_2}$, where $A_1 \sim \chi^2_{t_{ij}-1}$ and $A_2 \sim \chi^2_{t_{ij}-1}$. Since $Z_{ij}$ is the ratio of two independent chi-squared random variables, $Z_{ij} \sim \beta'\left(\frac{t_{ij}-1}{2},\frac{t_{ij}-1}{2}\right)$. Thus,\\ $w.p\ \ F_{\beta'\left(\frac{t_{ij}-1}{2},\frac{t_{ij}-1}{2}\right)}\!\!\left(\frac{c\!+\! \sqrt{c^2\!-\!4}}{2}\right)- F_{\beta'\left(\frac{t_{ij}-1}{2},\frac{t_{ij}-1}{2}\right)}\!\!\left(\frac{c\!- \!\sqrt{c^2\!-\!4}}{2}\right)\!,$
		\begin{equation}\label{ebndsecond}
\log_{e}\left(\frac{V_{ij}}{4}+\frac{1}{2}\right) \leq \frac{2}{\sqrt{t_{ij}-1}}.
		\end{equation}
		Note that $d_{ij} = \frac{1}{4}\!\left(U_{ij}+ \log_{e}\left[\frac{V_{ij}}{4}+\frac{1}{2}\right]\right).$ We know that $\mathbb{P}(A\cap B) \geq \mathbb{P}(A)+ \mathbb{P}(B) - 1$. So, combining (\ref{et1bnd}) and (\ref{ebndsecond}),  $w.p\ \geq 1 - \epsilon_{t_{ij}}$
		with $\epsilon_{t_{ij}}$ as defined in the statement,
		\begin{align*}
		d_{ij} \leq \frac{1}{4}\Big(\frac{2}{\sqrt{t_{ij}-1}} + \frac{2}{\sqrt{t_{ij}-1}}\Big) =  \frac{1}{\sqrt{t_{ij}-1}}.\tag*{\qedhere}
		\end{align*}
	\end{proof}}
	{Table \ref{table1} gives a numerical perspective of the bound and its probability. As seen, 
	the lower bound on probability 
	increases with $t_{ij}$
	.}
	\begin{table}[h!]
	\caption{Probabilities  and Bounds in Theorem \ref{t1}\vspace{-0.125cm}}
	\label{table1}
	\begin{center}
		\begin{tabular}{?c?c|c|c|c?} 
			\specialrule{1pt}{0pt}{0pt}
			$t_{ij}$                &   $11$  &   $51$  &  $101$  &  $151$  \\
			\specialrule{1pt}{0pt}{0pt}
			$\frac{1}{\sqrt{t_{ij}-1}}$             & 0.3162  & 0.1414  &   0.1   & 0.0816  \\ 
			\specialrule{1pt}{0pt}{0pt}
			$1-\epsilon_{t_{ij}}$& 0.970174 & 0.999567 & 0.999980	& 0.999998\\
			\specialrule{1pt}{0pt}{0pt}
		\end{tabular}
	\end{center}
\end{table}
	\par 	{Through the next theorem, we will derive a lower bound for $d_{ij}$, when $I_{i}$ and $I_{j}$ contain points from different subspaces.} 
	{
		\begin{theorem}\label{t2}
		Suppose, $I_i$ contain points only from subspace $\mathcal{U}_a$ and $I_j$ from $\mathcal{U}_b$ with $t_{ij}$ independent angle samples used for estimating the sample mean and variances. Let $M_{ab} = \dfrac{|\nu_a - \nu_{ab}|}{\sqrt{\rho_a^2+\rho_{ab}^2}}$,  $R_{ab} = \dfrac{\rho_a^2}{\rho_{ab}^2} + \dfrac{\rho_{ab}^2}{\rho_a^2}$ and $\alpha_{t_{ij}} = e^{\frac{4}{\sqrt{t_{ij}-1}}}$. 
		Then, 
		\begin{equation}\label{ediffss}
		d_{ij} \geq \frac{1}{\sqrt{t_{ij}-1}}\qquad w.p 
		\ \geq 1- \delta_{t_{ij}}^{ab},
		\end{equation}
		where
		\begin{align*}
			&\delta_{t_{ij}}^{ab} \! =\! 1\!-\! \Big[\! F_{\chi^2_{t_{ij}-1}}\!\big((t_{ij}\!-\!1)\alpha_{t_{ij}}\big)\! -\! F_{\chi^2_{t_{ij}-1}}\!\big((t_{ij}\!-\!1)(2\!-\!\alpha_{t_{ij}})\big)\!\Big]^2\\&\times\Big[1 -F_{\chi^2(1,t_{ij}M_{ab}^2)}\big(t_{ij}\log_{e}(1+M_{ab}) \alpha_{t_{ij}}\big)\Big], \mbox{ when }
		\end{align*}
		\begin{equation}\label{etmin}
		t_{ij} \geq 1+ \frac{16}{(\log_e\psi_{ab})^2},
		\end{equation}
		where,\\
		$\psi_{ab} = \frac{\sqrt{(R_{ab}-2)^2(1+M_{ab})^2+32R_{ab}(1+M_{ab})} -(R_{ab}-2)(1+M_{ab})}{8}$.  
	\end{theorem}
	\begin{proof}
	Please refer to Appendix \ref{appendix2}
	\end{proof}
	Note that the bound in (\ref{etmin}) is a very conservative sufficient condition. A numerical perspective can be seen in Table \ref{Tablethm2}, where $t_{min} = \left\lceil 1+ \frac{16}{(\log_e\psi_{ab})^2}\right\rceil$. $R_{ab}$  and $M_{ab}$ give a sense of how well the distributions of angles are separated. As expected, when these quantities are large, $t_{min}$ reduces, i.e. with a lesser number of angle samples, we get a larger $d_{ij}$ value. One could also note that, for $0<M_{ab} <e-1$, the probabilities are not very high at $t_{ij} = t_{min}$, in which case the high probability condition of $M_{ab} = 0$ is applicable since $U_{ij} \geq 0$ in any case.}
	\par 	{Suppose we have a clustering $\mathcal{C}_K$, such that each constituent cluster in $\mathcal{C}_K$ only contains points from the same subspace. Let $S^{\mathcal{O}}_K =\{(i,j)\ |\  \forall p \in I_i \text{ and } q \in I_j, \mathbf{x}_p \in \mathcal{U}_a \text{ and } \mathbf{x}_q \in \mathcal{U}_b, \ a\neq b, \ i,j \in 1,2,\ldots,K \}$ denote the set of all clustering index pairs such that the points in them belong to different subspaces and let $S^{\mathcal{I}}_K =\{(i,j) \text{	}| \text{	} \forall p \in I_i \text{ and } q \in I_j,\ \mathbf{x}_p,\ \mathbf{x}_q \in \mathcal{U}_a \text{	for some }a,\text{	} i,j \in 1,2,\ldots,K \}$ denote clustering index pairs such that the points in them belong to one subspace. Let $d_{\mathcal{O}} = \underset{(i,j) \in S^{\mathcal{O}}_{K}}{\min}\text{	}d_{ij}$  be the minimum inter-subspace distance and suppose $t_\mathcal{O}$ angle samples were used for its computation. Let the corresponding closest subspaces be $\mathcal{U}_a$ and $\mathcal{U}_b$. Let $T^{\mathcal{I}}_K = \{t_{ij}  \ | \ (i,j) \in S^{\mathcal{I}}_K\}$ denote the set of the number of independent angle samples used for computation of the distance between cluster pairs which contain points from only one subspace. Under the above notations and assumptions, we can state the following corollary, which is a direct consequence of Theorems \ref{t1} and \ref{t2} and defines the threshold $\zeta_{K}$.
		\begin{table}[t!]
			\begin{center}
				\caption{Probabilities and Bound in theorem \ref{t2}: $t_{min} = \left\lceil1+ \frac{16}{(\log_e\psi_{ab})^2}\right\rceil$}
				\label{Tablethm2}
				\setlength{\tabcolsep}{1pt}
				\begin{tabular}{?c?c|c?c|c?c|c?}
					\clineB{2-7}{2}
					\multicolumn{1}{c?}{}&\multicolumn{2}{c?}{$M_{ab}=0$}&\multicolumn{2}{c?}{$M_{ab}=2$}&\multicolumn{2}{c?}{$M_{ab}=3$} \\
					\specialrule{1pt}{0pt}{0pt}
					$R_{ab}$&$t_{min}$ & $1- \delta_{t_{min}}^{ab}$ &$t_{min}$ & $1- \delta_{t_{min}}^{ab}$&$t_{min}$ & $1- \delta_{t_{min}}^{ab}$\\
					\hline
					3&$\ \ \; \,$1575$\ \ \;\,$&0.994042&$\ \  \ \;\,$50$\ \  \ \;\, $&0.998565&$\ \ \; \ \,$35$\ \ \;\ \,$&0.998918\\
					10&118&0.997716&40&0.998573&35&0.998918\\
					$\ \ $20$\ \ $&68&0.998275&38&0.998440&35&0.998918\\
					\specialrule{1pt}{0pt}{0pt}
				\end{tabular}
			\end{center}
		\end{table}}
	{
		\begin{corollary}[To Theorems \ref{t1} and \ref{t2}]\label{c1}
		Suppose the clustering $\mathcal{C}_K$ as described above, with mergeable pair $(I_{i^*}, I_{j^*})$, has a non-empty $S^{\mathcal{I}}_K$. Let $t_K$ be the number of independent samples used for estimates in the mergeable pair. If $t_{\mathcal{O}} \geq 1+ \frac{16}{(\log_e\psi_{ab})^2}$ and $t_{\mathcal{O}} \leq$ at least one element in $T^{\mathcal{I}}_K$, then $w.p \geq 1 - \epsilon_{t_{K}} - \delta_{t_{\mathcal{O}}}^{ab}$, the mergeable pair contains points from the same subspace and $\gamma_K \leq \zeta_K$, when 
		\begin{equation}\label{ethr}
		\zeta_K = \frac{1}{\sqrt{t_K-1}}.
		\end{equation}
	\end{corollary}
	\begin{proof}
		As per Theorem \ref{t1}, for any $(i,j) \in S^{\mathcal{I}}_K$, $d_{ij} \leq \frac{1}{\sqrt{t_{ij}-1}}$. From Theorem \ref{t2}, for the closest inter-subspace cluster pair $d_{\mathcal{O}}\geq \frac{1}{\sqrt{t_{\mathcal{O}-1}}}\ w.p\ \geq 1-\delta_{t_{\mathcal{O}}}^{ab}$. Also, we have assumed that $t_{\mathcal{O}} \leq $ at least one element in $T^{\mathcal{I}}_K$, which means that there exist some $t_{ij}$ with $(i,j) \in S^{\mathcal{I}}_K$, such that $d_{\mathcal{O}} \geq \frac{1}{\sqrt{t_{ij}-1}}$. Hence, $d_{\mathcal{O}} \geq d_{ij}$ for some $(i,j) \in S^{\mathcal{I}}_K$. So the mergeable pair $(I_{i^*}, I_{j^*})$ contains points from the same subspace and hence $d_{i^*j^*} \leq \frac{1}{\sqrt{t_K-1}}$ with probability $1 - \epsilon_{t_{K}}$ from Theorem \ref{t1}. Hence, $w.p\ \geq 1 - \epsilon_{t_{K}} - \delta_{t_{\mathcal{O}}}^{ab}$, the statement is true.
	\end{proof}
	From the above, it is clear that, at any stage $K$ of the algorithm, if suitable conditions are satisfied, 
	Algorithm \ref{alg_merge} merges two clusters from the same subspace and hence the merging process keeps merging constituent clusters with only indices of points from the same subspace until it reaches a true clustering. Through Corollary \ref{c1}, we have shown that until then $\gamma_K \leq \zeta_{K}$ with a high probability.} And at this point, since the mergeable pair will contain points from different subspaces, $\gamma_K$ will be high, and we select this crossover instance as the estimated clustering in Algorithm \ref{alg_zeta}. We demonstrate this behaviour of $\gamma_{K}$ in Fig. \ref{fig_gamma_K}. Here, we have considered Model \ref{model_ss} and the data points are drawn from $L=6$ subspaces, each of dimension $7$ in $\mathbb{R}^{100}$. It can be seen that $\gamma_{K}$ is close to $0$ when $K>6$ and when $K=6$, $\gamma_{K}$ is bounded away from $0$ and $\gamma_{K}>\zeta_{K}$. From the above observations, we can state the following remark:
	\begin{center}
		\begin{figure}[h!]
			\includegraphics[width=\linewidth]{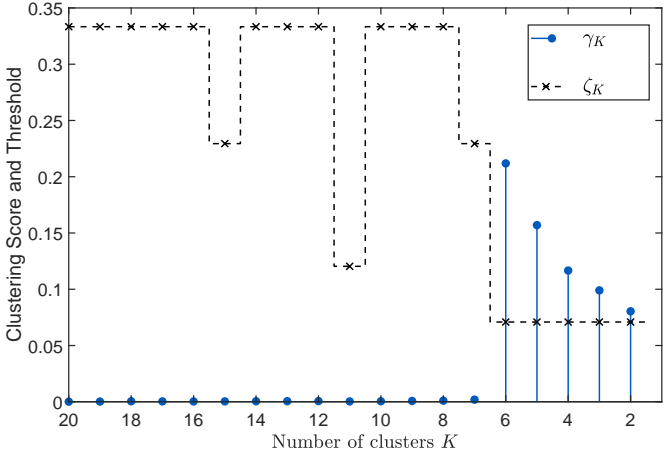}
			\caption{Clustering score $\gamma_{K}$ and threshold $\zeta_{K}$ in Model \ref{model_ss} with $L=6$ subspaces in $\mathbb{R}^{100}$ each is of dimension 7}
			\label{fig_gamma_K}
		\end{figure}
	\end{center}
	\begin{remark}
	Suppose we are given an initial clustering $\mathcal{C}_P$
	such that each constituent cluster in $\mathcal{C}_P$ contains only indices
	of points from the same subspace. {Under suitable conditions as described in Theorems \ref{t1} and \ref{t2} and Corollary \ref{c1}}, an estimate $ \hat{\mathcal{C}}_{\hat{L}}$, which is arrived at by the merging process in Algorithm \ref{alg_merge} and the selection process in Algorithm \ref{alg_zeta}, is a true clustering with a high probability.
	\end{remark}

	{Note that the conditions derived are very conservative sufficient conditions for the algorithm to merge correctly 
	at any stage. For instance, consider Model 1 where $M_{ab} = 0$. 
	It has been observed that even for the cases with smaller $R_{ab}$, the algorithm works perfectly, starting with initial clusters having just 3 or 4 points (i.e., $t_{ij}=2$). With $n=100$ and 
	ranks of the subspaces 
	$r=10$, 
	$R_{ab}
	=\frac{n-2}{r-2}+\frac{r-2}{n-2}
	=12.33$ and Theorem \ref{t2} 
	dictates $t_{min}=96$ 
	to achieve $1-\delta_{t_{min}}^{ab}=0.9979$. But with just 
	$t_{ij}=2$, the algorithm achieves error-free clusterings, 
	as shown in Table \ref{tab_combined1}
	. The error is also minimum in many real 
	datasets 
	as illustrated in Section \ref{sNumerical_validations}
	. Thus, the algorithm works with minimal error in much harsher conditions than those suggested by Theorems  \ref{t1} and \ref{t2} and Corollary \ref{c1}. 
	It is extremely difficult to get tighter bounds and conditions as the distribution of empirical Bhattacharyya distance has not been characterized and to do so is beyond the scope of this work. }
	\subsection{A Note on Complexity}
	The complexity for computing all the $N \choose 2$ angles between $n$-dimensional points is $O(N^2n)$. For $P$ constituent clusters in the initial clustering, one has to compute the statistical distances for all possible combinations. When these are computed as described in the previous sections, one could precompute the sum of angles and the squared sum of angles for all possible combinations, leading to a complexity of $O(P^2)$. In the merging process, using the precomputed values, one can update the distances using simple arithmetic, leading to a complexity of $O(P)$. Hence, the overall complexity is $O(max(N^2n,P^2))$. Runtime comparisons are provided in Section \ref{valid_results}.
	\section{Numerical Results}\label{sNumerical_validations}
	\begin{table*}[h!]
		\begin{center}
			\caption{Performance of subspace clustering algorithms on synthetic datasets}
			\label{tab_combined1}
			\setlength{\tabcolsep}{3pt}
			\begin{tabular}{?cc|c?c|c?c|c?c|c?c|c?c|c?c|c?c|c?c|c?c|c?}
				\clineB{4-21}{2}
				\multicolumn{3}{c?}{}&\multicolumn{2}{c?}{\textbf{Ours}}&\multicolumn{2}{c?}{\textbf{SSC-ADMM}}&\multicolumn{2}{c?}{\textbf{SSC-OMP}}&\multicolumn{2}{c?}{\textbf{EnSC}}&\multicolumn{2}{c?}{\textbf{BDR-Z}}&\multicolumn{2}{c?}{\textbf{TSC}}&\multicolumn{2}{c?}{\textbf{ALC}}&\multicolumn{2}{c?}{\textbf{LRR}}&\multicolumn{2}{c?}{\textbf{$\lambda$-Means}}\\
				\cline{4-21}
				\multicolumn{3}{c?}{} &\multicolumn{1}{c|}{\textbf{CE}}&\multicolumn{1}{c?}{\textbf{NMI}}&\multicolumn{1}{c|}{\textbf{CE}}&\multicolumn{1}{c?}{\textbf{NMI}}&\multicolumn{1}{c|}{\textbf{CE}}&\multicolumn{1}{c?}{\textbf{NMI}}&\multicolumn{1}{c|}{\textbf{CE}}&\multicolumn{1}{c?}{\textbf{NMI}}&\multicolumn{1}{c|}{\textbf{CE}}&\multicolumn{1}{c?}{\textbf{NMI}}&\multicolumn{1}{c|}{\textbf{CE}}&\multicolumn{1}{c?}{\textbf{NMI}}&\multicolumn{1}{c|}{\textbf{CE}}&\multicolumn{1}{c?}{\textbf{NMI}}&\multicolumn{1}{c|}{\textbf{CE}}&\multicolumn{1}{c?}{\textbf{NMI}}&\multicolumn{1}{c|}{\textbf{CE}}&\multicolumn{1}{c?}{\textbf{NMI}}\\
				\clineB{4-21}{2}
				\specialrule{1pt}{2pt}{0pt}
				\multicolumn{2}{?c|}{\multirow{3}{*}{ \spheading{\centering \textbf{Subspace-Normal}}}}&$L=4$ &\textbf{0}&\textbf{1}&\textbf{0}&\textbf{1}&0.002&0.997&0.001&0.998&0.022&0.985&\textbf{0}&\textbf{1}&\textbf{0}&\textbf{1}&0.066&0.973&NA&NA\\
				{}&{} &$L=7$ &\textbf{0}&\textbf{1}&0.001&0.999&0.101&0.949&0.062&0.969&0.001&0.999&\textbf{0}&\textbf{1}&\textbf{0}&\textbf{1}&0.074&0.969&NA&NA\\
				{}&{} &$L=10$ &\textbf{0}&\textbf{1}&0.02&0.989&0.11&0.954&0.126&0.95&0.024&0.989&0.02&0.99&\textbf{0}&\textbf{1}&0.12&0.961&NA&NA\\
				
				\clineB{1-21}{2}
				\specialrule{1pt}{2pt}{0pt}	
				\multicolumn{2}{?c|}{\multirow{3}{*}{\spheading{\centering \textbf{Subspace-Uniform}}}}&$L=4$ &\textbf{0}&\textbf{1}&0.023&0.984&0.001&0.998&0.002&0.997&0.025&0.988&\textbf{0}&\textbf{1}&\textbf{0}&\textbf{1}&0.034&0.967&NA&NA\\
				{}&{} &$L=7$ &\textbf{0}&\textbf{1}&0.024&0.987&0.023&0.987&0.024&0.987&0.024&0.987&\textbf{0}&\textbf{1}&\textbf{0}&\textbf{1}&0.04&0.977&NA&NA\\
				{}&{} &$L=10$ &\textbf{0}&\textbf{1}&0.04&0.979&0.02&0.989&0.021&0.989&0.04&0.979&0.02&0.989&\textbf{0}&\textbf{1}&0.005&0.998&NA&NA\\
				\clineB{1-21}{2}
				\specialrule{1pt}{2pt}{0pt}	
				\multicolumn{2}{?c|}{\multirow{3}{*}{\spheading{\centering \textbf{Subspace-Dependent}}}}&$L=12$ &\textbf{0}&\textbf{1}&0.078&0.969&0.089&0.955&0.077&0.969&0.124&0.952&0.078&0.97&0.917&0&0.319&0.901&NA&NA\\
				{}&{} &$L=16$&\textbf{0}&\textbf{1}&0.023&0.99&0.097&0.962&0.08&0.982&0.143&0.958&0.053&0.981&0.937&0.005&0.186&0.946&NA&NA\\
				{}&{} &$L=20$ &\textbf{0}&\textbf{1}&0.021&0.99&0.115&0.961&0.09&0.984&0.147&0.956&0.07&0.976&0.95&0.003&0.162&0.965&NA&NA\\
				\clineB{1-21}{2}
				\specialrule{1pt}{2pt}{0pt}
				\multicolumn{2}{?c|}{\multirow{3}{*}{\spheading{\centering \textbf{DP Process}}}}&${\rho}/{\sigma} = 1$ &\textbf{0.0007}&$\,$\textbf{0.975}$\,$&0.392&0.007&0.725&0.006&0.725&0.006&0.645&0.006&0.581&0.006&0.281&0&0.282&0.006&0.31&0\\
				{}&{} &${\rho}/{\sigma} = 5$ &\textbf{0.0006}&\textbf{0.954}&0.438&0.007&0.74&0.007&0.74&0.007&0.661&0.006&0.619&0.007&0.327&0&0.328&0.006&0.306&0\\
				{}&{} &${\rho}/{\sigma} = 9$ &\textbf{0.0008}&\textbf{0.972}&0.411&0.008&0.729&0.005&0.729&0.005&0.652&0.005&0.594&0.006&0.321&0&0.323&0.007&0.007&0.949\\
				\clineB{1-21}{2}
				\specialrule{1pt}{0pt}{0pt}
			\end{tabular}
		\end{center}
	\end{table*}

	In this section, we validate the performance of our algorithm numerically through simulations in synthetic as well as real datasets. We compare the performance with state-of-the-art subspace clustering algorithms - {SSC-ADMM\cite{elhamifar2013sparse}, SSC-OMP\cite{you2016scalable}, EnSC\cite{you2016oracle}, TSC\cite{heckel2015robust}, ALC\cite{ma2007segmentation} and LRR\cite{liu2010robust} 
	and another recent algorithm BDR-Z\cite{lu2019subspace}}. The codes for these algorithms are obtained from respective authors. All these algorithms, except ALC, require us to provide an estimate of the number of clusters, $L$ apriori, which we denote as $L_{in}$. {However, as mentioned in \cite{heckel2015robust}, we can estimate $L$ by eigen-gap heuristic - we denote $L$ estimated using this technique as $\hat{L}_{eg}$. We also denote the number of clusters estimated by our algorithm as $\hat{L}_{our}$.  Throughout the experiments, for the existing algorithms, we have used the default or tuned parameters provided by the authors in their codes, unless stated explicitly}
	. For our algorithm, we use the initial clustering described in Section \ref{sinit_clust} unless otherwise specified. The best performance in each experiment is given in \textbf{boldface}.
	
	\subsection{Metrics for comparison}
	
	We compare the performance in terms of \textit{Clustering Error} (CE) and \textit{Normalized Mutual Information} (NMI). 
	Clustering error \cite{heckel2015robust} is defined as the fraction of points misclassified by the algorithm. It is computed as follows. Let $c_i$ and $\tilde{c}_i,\ i=1,2,\ldots N$ denote respectively the true cluster label of the point $\mathbf{m}_i$ and the label assigned to it by the algorithm. Then,
	\begin{equation*}\label{e_ce}
	\mbox{CE}=\min_{\pi}\bigg(1-\dfrac{1}{N}\sum_{i=1}^{N}\mathbb{I}\left(c_i,\pi(\tilde{c}_i)\right)\bigg),
	\end{equation*}
	where $\mathbb{I}(a,b)=1$ if $a=b$, $0$ otherwise. $\pi(\tilde{c}_i)$ is the one-one reassignment of the label $\tilde{c}_i$, $\pi(\tilde{c}_i)\in\{1,2,\ldots ,L\}$ where $\tilde{c}_i\in\{1,2,\ldots ,\hat{L}\}$. We compute Normalized Mutual Information \cite{strehl2002cluster} as
	\begin{equation*}\label{e_nmi}
	\mbox{NMI}=\dfrac{\mathcal{I}(\mathcal{C}_L;\hat{\mathcal{C}}_{\hat{L}})}{0.5\big(\mathcal{H}(\mathcal{C}_L)+\mathcal{H}(\hat{\mathcal{C}}_{\hat{L}})\big)}\, ,
	\end{equation*} 
	where $\mathcal{H}(\cdot)$ and $\mathcal{I}(\cdot ;\cdot)$ respectively denote the empirically computed entropy of the cluster and mutual information between clusters. CE close to 0 and NMI close to 1 are desirable.
	\subsection{Results on Synthetic Datasets}
	\subsubsection{Random Subspace Models} We first illustrate the results if the data points are from Model \ref{model_ss}. It is known that for the points $\mathbf{x}_i$'s to be uniformly distributed in $\mathcal{U}_k\cap\mathbb{S}^{n-1}$, the individual coordinates of $\mathbf{m}_i$'s have to sampled independently from a standard normal distribution \cite{cai2013distributions}. We call this dataset as `Subspace-Normal'. We also show the results in the random subspace model when we sample individual coordinates of $\mathbf{m}_i$'s from a standard uniform distribution. We call this one as `Subspace-Uniform'.
	\begin{table}[h!]
		\begin{center}
			\caption{Error on estimated number of clusters on Synthetic datasets}
			\label{tab_ss_L_est}
			\setlength{\tabcolsep}{7pt}
			\begin{tabular}{?c?c?c|c|c?}
				\clineB{2-5}{2}
				\multicolumn{1}{c?}{} &\multicolumn{1}{c?}{\centering \textbf{$|\boldsymbol{L}-\boldsymbol{\hat{L}}|$}} & \textbf{Ours} & \textbf{Eigen-gap} &  \textbf{ALC} \\
				\clineB{2-5}{2}
				\specialrule{1pt}{2pt}{0pt}
				{\multirow{3}{*}{\spheading{\centering \textbf{Subspace-Normal}}}}&Mean &       \textbf{0}   & 13.23   & \textbf{0}\\
				{}&Median & 0 & 0 &0\\
				{}&Std &    0     &    93.43 &    0 \\
				\clineB{1-5}{2}
				\specialrule{1pt}{2pt}{0pt}
				{\multirow{3}{*}{\spheading{\centering \textbf{Subspace-Uniform}}}}&Mean &       \textbf{0}   & 14.94   & \textbf{0}\\
				{}&Median & 0 & 0 &0\\
				{}&Std &    0     &    84.387  &    0 \\
				\clineB{1-5}{2}
				\specialrule{1pt}{2pt}{0pt}
				{\multirow{3}{*}{\spheading{\centering \textbf{Subspace-Dependent}}}}&Mean &       \textbf{0}   & 10.415   & 13.25\\
				{}&Median & 0 & 0 &13.25\\
				{}&Std &    0     & 28.84  &  0\\
				\specialrule{1pt}{0pt}{0pt}
			\end{tabular}
		\end{center}
	\end{table} 
	For both these scenarios, we have taken 1000 points in $\mathbb{R}^{100}$, show results for $L = 4,7,10$ with a roughly equal number of points in each $L$. The dimension of each subspace is chosen to be 10. The results shown in Table \ref{tab_combined1} are averaged over 50 trials for each L. {For existing algorithms (except ALC) we provide the number of clusters estimated from eigen-gap heuristic as input, i.e., $L_{in}=\hat{L}_{eg}$.} {For ALC, we set have tried several values for $\epsilon$ and $\epsilon=1$ gave perfect recovery of subspaces in all the trials. We use the same $\epsilon$ throughout the remainder of the paper}. {In Table \ref{tab_combined1}, we also show results for dependent subspaces. Here, we choose $100$ basis vectors for $\mathbb{R}^{100}$. 
	Then, for each subspace, we choose $r_i = 10$ basis functions randomly from the collection and form $L=12,16,20$ such subspaces. Since $r_i \times L$ is greater than or close to $n$, we are bound to end up with subspaces which share common bases (more than 1 on many occasions), making them dependent. The data points from these subspaces are formed by the linear combination of the basis where the coordinates are randomly chosen from a standard uniform distribution.} 
	
	{In the subspace model, almost all the algorithms with $\hat{L}_{eg}$ supplied as input perform fairly well, while our algorithm achieves perfect clustering every time, with TSC achieving perfect clustering in most trials. ALC achieves perfect clustering for independent subspaces for all $L$ values but fails for dependent subspaces with the same $\epsilon$ value, where it always ends up with one cluster. In dependent subspaces, the performance degrades considerably for many algorithms like LRR and ALC, while it degrades marginally for others. In all the cases, our algorithm clusters perfectly. It is evident that the quality of the estimate of $L$, $\hat{L}_{eg}$ determines the success of the algorithms. In Table \ref{tab_ss_L_est}, we show the absolute error in the estimate of $L$, using eigen-gap, ALC, as well as our method, where the values are averaged over all the cases in Table \ref{tab_combined1} for all the synthetic datasets except DP process. Our method estimates the number correctly in all trials, while eigen-gap estimates it correctly for most trials, but overestimates the number by a very large value for a few of the trials as seen by the large values for the mean and standard deviation of the error with the median remaining at 0. ALC estimates it correctly with 0 error for independent subspaces. However, it fails for the case of dependent subspaces.}\\ 
	\subsubsection{Dirichlet Process Model} To illustrate the versatility of our algorithm in adapting to other clustering data models which are not subspace models, we show the results when the data points are obtained from Dirichlet Process (DP)\cite{kulis2011revisiting}. The results are also summarized in Table \ref{tab_combined1}. For each of the listed $\rho/\sigma$, we performed 50 trials, generating $1000$ points from $\mathbb{R}^{100}$ each time. Here, $\rho$ represent spread of cluster centroids, and $\sigma$ represent spread of points around their respective centroid. Larger $\rho/\sigma$ signifies widely separated dense clusters. For this dataset, we also include results from $\lambda$-means algorithm\cite{comiter2016lambda}, a parameter tuning algorithm developed for DP model data. It tunes for the parameter $\lambda$ in DP-means algorithm\cite{kulis2011revisiting}. For each $\rho/\sigma$, we tune for $\lambda$ in the first trial and use that value for remaining trials. It is observed that $\lambda$-means performs badly when $\rho/\sigma$ is small. This is because it is a distance-metric based algorithm and at smaller $\rho/\sigma$, clusters are not well separated. {For the case of DP dataset, we provide the true number of clusters $L$, as input to all the algorithms instead of eigen-gap estimates. Even then, the performances of existing algorithms are poor as evident from Table \ref{tab_combined1}. This shows the lack of adaptability to a non-subspace model for clustering}. We can observe that our algorithm performs consistently better for all $\rho/\sigma$ with mean CE $\leq0.08\%$ and mean NMI$\geq0.954$.
	\begin{table}
		\begin{center}
			\caption{Details of some real datasets}
			\label{tab_dst_dset}
			\setlength{\tabcolsep}{3pt}
			\begin{tabular}{?c? *{2}{c|}c?}
				\specialrule{1pt}{0pt}{0pt}
				
				{\textbf{\centering Dataset                                            }} & {$\quad\,\boldsymbol{n}\quad\,$} & {$\quad\,\boldsymbol{N}\quad\,$} & {$\quad\,\boldsymbol{L}\quad\,$} \\
				\specialrule{1pt}{0pt}{0pt}
				\specialrule{1pt}{1pt}{0pt}
				{Gene Expression Cancer RNA-Seq\cite{weinstein2013cancer}\cite{Dua:2019}} &               16383              &                801               &                 5                \\
				{Novartis multi-tissue\cite{novartis}                                   } &                1000              &                103               &                 4                \\
				{Wireless Indoor Localization\cite{rohra2017user}\cite{Dua:2019}        } &                 7                &               2000               &                 4                \\
				{Phoneme\cite{hastie1995penalized}                                      } &                256               &                4509              &                 5                \\
				{MNIST\cite{lecun1998gradient}                                                      } &                784               &       40000                     &                 10               \\
				{Extended Yale-B\cite{GeBeKr01,KCLee05}                                 } &                32256             &                2432              &                38               \\
				{Hopkins-155\cite{tron2007benchmark}}&30-200&39-556&2,3\\
				\specialrule{1pt}{0pt}{0pt}
				
			\end{tabular}
		\end{center}
	\end{table}
	\begin{table*}
		\begin{center}
			\caption{Performance of subspace clustering algorithms on some real datasets}
			\label{tab_real_ce}
			\setlength{\tabcolsep}{11pt}
			\begin{tabular}{?c? *{3}{Y|Y?}Y|Y?Y|Y?}
				\specialrule{1pt}{0pt}{0pt}
				\multirow{2}{1.25cm}[-1.1em]{\textbf{Algorithm}} &  \multicolumn{2}{L?}{\textbf{Gene Expression Cancer RNA-Seq}} & \multicolumn{2}{L?}{\textbf{Novartis multi-tissue}} & \multicolumn{2}{L?}{\textbf{Wireless Indoor Localization}}  & \multicolumn{2}{L?}{\textbf{Phoneme}}\\
				\cline{2-9}
				{                                              } &  \textbf{CE }  & \textbf{NMI}    &  \textbf{CE }   & \textbf{NMI}    &  \textbf{CE }   & \textbf{NMI}    &  \textbf{CE }   & \textbf{NMI}        \\
				\specialrule{1pt}{0pt}{0pt}
				\specialrule{1pt}{1pt}{0pt}
				Ours                       & \textbf{0.0087}& \textbf{0.9769} & \textbf{0.1845} &     \textbf{0.7247 }     & \textbf{0.1720} & \textbf{0.7510} & \textbf{0.2790} &      0.6222          \\
				\hline
				SSC-ADMM                     &    0.0724      &      0.9363     &    0.8058      &      0.5814        &      0.9975     &      0.3085     &      0.3387     &      0.7688           \\
				\hline
				SSC-OMP                     &     0.0849     &    0.8716       &    0.8155      &      0.5042        &      0.9975     &      0.3085     &      0.5445     &      0.2571         \\
				\hline
				EnSC                       &    0.0674      &    0.9386       &     0.8058     &       0.5896       &      0.9975     &      0.3085     &      0.3251     & \textbf{0.7688}        \\
				\hline
				BDR-Z                      &    0.0587      &    0.9454       &     0.8058      &     0.5814       &      0.9980     &      0.3084     &      0.4400     &      0.3722             \\
				\hline
				TSC                       &    0.0612      &    0.9441       &     0.8058      &       0.6210      &      0.9975     &      0.3085     &      0.3227    &      0.7629             \\
				\hline
				ALC                       &       0.7079      &        0.6298      &     0.7282      &       0.3469       &      0.7500     &      0          &      0.8960        &      0.4390                \\
				\hline
				LRR                       &     0.0637     &       0.9168    &    0.6311       &       0.6557       &      0.9975     &      0.3085     &      0.4424     &      0.5682             \\
				\hline
				\specialrule{1pt}{0pt}{0pt}
				
			\end{tabular}
		\end{center}
	\end{table*}
	
	\subsection{Results on Real Datasets}\label{sRealDatasets}
	We illustrate the performances of our algorithm and other subspace clustering algorithms on some real datasets. Table \ref{tab_dst_dset} provides the details of the datasets we have used. The first two datasets are gene expression datasets. In such applications, the number of clusters in the dataset may not be known apriori. 
	{We have also performed experiments on popular image datasets, namely MNIST \cite{lecun1998gradient} and extended Yale-B \cite{GeBeKr01,KCLee05} as well as the motion segmentation problem in Hopkins-155 dataset \cite{tron2007benchmark}. The results for the first four datasets are given in Table \ref{tab_real_ce}, where all the existing algorithm except ALC are provided with $L_{in}=\hat{L}_{eg}$. For ALC, we provide $\epsilon=1$. The comparison of estimates of the number of clusters using our algorithm with eigen-gap and ALC for various datasets are given in Table \ref{tab_real_L_est}.}
	
	{From Table \ref{tab_real_ce}, we can see that our algorithm outperforms all the other algorithms in most of the cases for these datasets. This can be interpreted as follows. The eigen-gap estimates for $L$, given in Table \ref{tab_real_L_est}, that we provide as an input to the other algorithms are not necessarily very accurate. This, along with the improper setting of tuning parameters, affects the algorithm performances in Table \ref{tab_real_ce}. Our algorithm by virtue of being non-parametric performs well across datasets without tuning or ground truth knowledge. Though our algorithm is developed for high dimensional data, its performance in wireless indoor localization dataset illustrates that we can also use it for low dimensional datasets. These results show the adaptability of the proposed algorithm across datasets of different types from diverse domains.}
	\begin{table}
		\begin{center}
			\caption{Estimated number of clusters in some real datasets}
			\label{tab_real_L_est}
			\setlength{\tabcolsep}{2pt}
			\begin{tabular}{?c? c|c| c?}
				\clineB{1-4}{2}
				\multirow{2}{*}{\textbf{Dataset}}&\multicolumn{3}{c?}{\textbf{Estimated number of clusters}} \\
				\cline{2-4}
				{                            } & \textbf{$\,\ \ \hat{\boldsymbol{L}}_{\boldsymbol{our}}\ \ \,$} & $\;$\textbf{ $\ \ \hat{\boldsymbol{L}}_{\boldsymbol{eg}}\ \ \ $}&$\;$\textbf{ $\hat{\boldsymbol{L}}_{\boldsymbol{ALC}}\;$}\\[-\arrayrulewidth]
				\clineB{1-4}{2}
				\specialrule{1pt}{2pt}{0pt}
				Gene Expression Cancer RNA-Seq &                                   6                                   &                                   6       &18                            \\
				Novartis multi-tissue          &                                   5                                   &                                   21    &1                              \\
				Wireless Indoor Localization   &                                   11                                  &                                  1999    &1                             \\
				Phoneme                        &                              \textbf{5}                               &                                   3 &45                                  \\
				MNIST                          &                                  \textbf{10}                                  &                                  39997     &-                             \\ 
				Extended Yale-B                &                                  39                                  &                                  2431        &269                          \\ 
				Hopkins-155 (2 objects)&$2.0083$&$118.32$ &$1.333$\\
				Hopkins-155 (3 objects)&$2.2$&$215.68$ &$1.314$\\
				\specialrule{1pt}{0pt}{0pt}
				
			\end{tabular}
		\end{center}
	\end{table}
	
	{The results in Table \ref{tab_real_L_est} reconfirm what we observed in Table \ref{tab_ss_L_est} that eigen-gap heuristic occasionally selects a very large number of clusters.  For instance, consider the wireless indoor localization dataset, the eigen-gap heuristic provides 1999 clusters, i.e., it considered each point (except a pair) as a cluster. This results in high CE for the algorithms using that estimate. Here, our algorithm outputs 11 clusters, and the CE is quite low, suggesting that the excess 4 clusters are smaller in size. Also, in Phoneme dataset, our method has recovered the exact number of clusters, and it predicted one additional cluster in Gene Expression Cancer RNA-Seq and Novartis multi-tissue datasets. Hence, our algorithm is better at finding the number of clusters in these datasets.}
	\subsubsection{Results on image datasets} {We have also performed experiments on three image datasets - the popular MNIST dataset \cite{lecun1998gradient}, face clustering using the extended Yale-B dataset \cite{GeBeKr01,KCLee05} and Hopkins-155 dataset \cite{tron2007benchmark} for motion segmentation. For each image in MNIST, we use extracted features from ScatNet\cite{bruna2013invariant} and then projected the features to dimension 500 using PCA and use them for all the algorithms. Due to memory limitations, we performed experiments by taking only 40000 data samples. Motivated from \cite{peng2018structured}, we obtained DSIFT features for extended Yale-B dataset and projected them to dimension 500 {using PCA} and then obtained results for all the algorithms. We use the dataset as it is for Hopkins-155.}
	
	{The parameters of EnSC are tuned for MNIST dataset, and those of SSC-ADMM are tuned to Hopkins. Note that eigen-gap heuristic seems to give extremely bad results in estimating the number of clusters (see Table \ref{tab_real_L_est}) in image datasets. Our method does not require the number of clusters apriori. However, it requires a good set of initial clusters. In Table \ref{tab_image}, all existing algorithms are given the true number of clusters. Otherwise, if the eigen-gap heuristic is used, the results would be extremely poor with CE close to 1 and NMI closer to 0. Hence to be also fair to our method, which does not know the true number of clusters, we provide it with a pure set of initial clusters. For MNIST, we provide 2000 initial clusters, for extended Yale-B we give 266 clusters, and for Hopkins, we use 5 points per initial cluster in each video.}
	
	{As seen from Table \ref{tab_image}, our method performs at par with the state-of-the-art in MNIST, using ScatNet features. Many algorithms for the large MNIST dataset are really slow. for instance, ALC did not converge even after days of running, and hence those results are not reported. {Our algorithm, however, even ran with the whole dataset (70000 points), without a problem. It had a CE of $0.0018$ and NMI of $0.9951$ for the whole MNIST dataset when provided with 2800 initial clusters.} For extended Yale-B using the feature extraction we described earlier, our algorithm outperforms all other methods {with the lowest CE and highest NMI, even when other algorithms are provided with the right number of clusters and the parameters tuned through grid search}. The influence of feature extraction for images on the success of the proposed method is discussed in detail in the next section. For the motion segmentation problem in Hopkins-155, our results are again at par with the state-of-the-art with only SSC-ADMM (which is tuned perfectly for this dataset) outperforming us marginally in terms of CE. }
	
	\begin{table}
		\begin{center}
			\caption{Performance of subspace clustering algorithms on image datasets}
			\label{tab_image}
			\setlength{\tabcolsep}{4pt}
			\begin{tabular}{?c?c? *{2}{c|} c?}
				\clineB{1-5}{2}
				\multirow{2}{*}{\textbf{Algorithm}} & \multirow{2}{*}{\textbf{Metric}} & \multicolumn{3}{c?}{\textbf{Dataset}} \\
				\cline{3-5}
				{                                   } & { } & \textbf{$\quad\, $MNIST$\quad\,$}  & \textbf{$\,$Ext. Yale-B$\,$} &  \textbf{Hopkins-155}      \\[-\arrayrulewidth]
				\clineB{1-5}{2}
				\specialrule{1pt}{2pt}{0pt}
				\multirow{2}{*}{\centering Ours}      &  CE &          \textbf{ 0.0015 }         &        \textbf{ 0.0074}      &            0.0938          \\
				{                                  }  & NMI &           \textbf{ 0.9959}         &        \textbf{ 0.9979 }     &            0.7406          \\
				\cline{1-5}{}
				\multirow{2}{*}{\centering SSC-ADMM}  &  CE &               0.1523               &              0.0185          &   \textbf{0.0479}          \\
				{                                  }  & NMI &               0.8540               &              0.9843          &   \textbf{0.8719}          \\
				\cline{1-5}{}
				\multirow{2}{*}{\centering SSC-OMP }  &  CE &               0.0532               &              0.0888         &            0.7087          \\
				{                                  }  & NMI &               0.8784               &              0.9748          &            0.0280          \\
				\cline{1-5}{}
				\multirow{2}{*}{\centering EnSC    }  &  CE &               0.0404               &              0.0465          &            0.2176          \\
				{                                  }  & NMI &               0.9085               &              0.9803          &            0.5603          \\
				\cline{1-5}{}
				\multirow{2}{*}{\centering BDR-Z   }  &  CE &               0.3496               &              0.0366          &            0.1531          \\
				{                                  }  & NMI &               0.5504               &              0.9874          &            0.7278          \\
				\cline{1-5}{}
				\multirow{2}{*}{\centering TSC     }  &  CE &               0.1650               &              0.0247          &            0.3735          \\
				{                                  }  & NMI &               0.8013                  &              0.9884         &            0.3673          \\
				\cline{1-5}{}
				\multirow{2}{*}{\centering ALC }      &  CE &                --                  &              0.3466       &            0.5900          \\
				{                                  }  & NMI &                --                  &              0.7942         &            0.2393          \\
				\cline{1-5}{}
				\multirow{2}{*}{\centering LRR     }  &  CE &               0.1831               &              0.3433          &            0.1246          \\
				{                                  }  & NMI &               0.8536               &              0.8810          &            0.7939          \\
				\clineB{1-5}{2}
				\specialrule{1pt}{0pt}{1pt}
				
			\end{tabular}
		\end{center}
	\end{table}
	\begin{table}
		\begin{center}
			\caption{Run time of algorithms on Ext. Yale-B dataset}
			\label{tab_r_time}
			\setlength{\tabcolsep}{3pt}
			\begin{tabular}{?c? c| c?}
				\specialrule{1pt}{0pt}{0pt}
				\multirow{2}{*}{$\ \ $\textbf{Algorithm}$\ \ $} & $\ $\textbf{No. of parameters}$\ $  & 	$\qquad $\textbf{Run time}$\qquad\ $\\
				{                                              }&    \textbf{(including $L_{in}$)}    &          \textbf{(in seconds)}          \\[-\arrayrulewidth]
				\clineB{1-3}{2}
				\specialrule{1pt}{2pt}{0pt}
				Ours                      &                  0                  &                   2.5              \\
				SSC-ADMM                  &                  2                  &                 21.3                 \\
				SSC-OMP                   &                  2                  &                  4.4                 \\
				EnSC                      &                  3                  &                   6.6              \\
				BDR-Z                     &                  3                  &                 444.5                \\
				TSC                       &                  2                  &                  7.6                 \\
				ALC                       &                  1                  &                   1397.1                    \\
				LRR                       &                  2                  &                   6267.1              \\
				\specialrule{1pt}{0pt}{0pt}
				
			\end{tabular}
		\end{center}
	\end{table}
	\begin{figure}[h]
		\begin{subfigure}{.5\textwidth}
			\centering
			\includegraphics[width=0.9\linewidth,height=4cm]{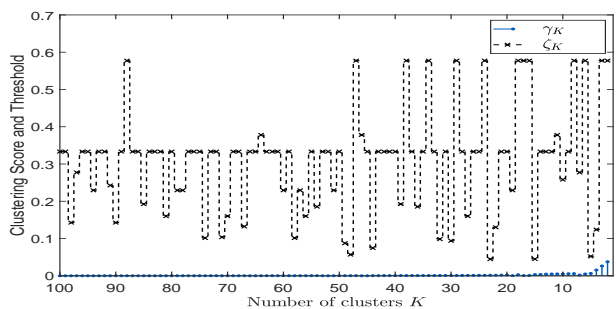}  
			\caption{With Raw pixels}
			\label{fig_yaleb_validation1}
		\end{subfigure}
		\begin{subfigure}{.5\textwidth}
			\centering
			\includegraphics[width=0.9\linewidth,height=4cm]{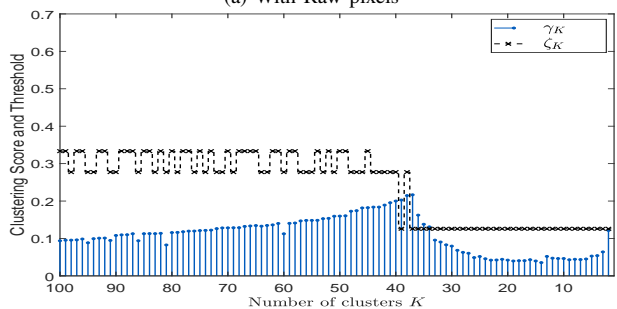}  
			\caption{With DSIFT features}
			\label{fig_yaleb_validation2}
		\end{subfigure}
		\caption{Clustering score $\gamma_{K}$ and threshold $\zeta_{K}$ in Ext. Yale-B dataset}
		\label{fig_yaleb_validation}
	\end{figure}
	\begin{figure}[h]
		\begin{subfigure}{.5\textwidth}
			\centering
			\includegraphics[width=0.9\linewidth,height=5cm]{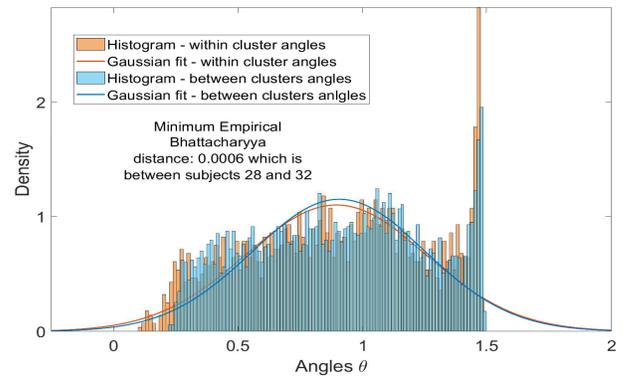}  
			\caption{With Raw pixels}
			\label{fig_yaleb_dist1}
		\end{subfigure}
		\begin{subfigure}{.5\textwidth}
			\centering
			\includegraphics[width=0.9\linewidth,height=5cm]{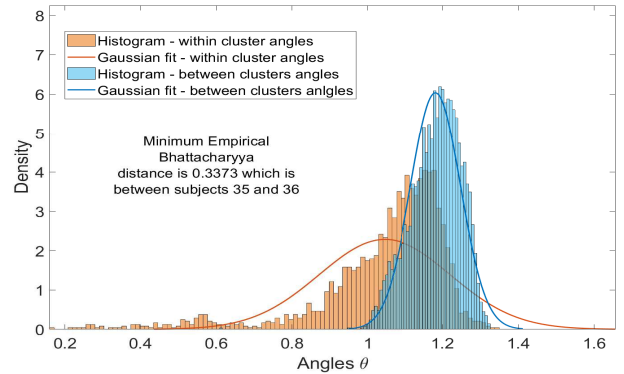}  
			\caption{With DSIFT features}
			\label{fig_yaleb_dist2}
		\end{subfigure}
		\caption{Distribution of within-cluster angles and between-clusters angles in Ext. Yale-B dataset  corresponding to minimum empirical Bhattacharyya distance }
		\label{fig_yaleb_dist}
	\end{figure}
	\section{Utility of our algorithm}\label{valid_results}
	{As stated previously, the proposed algorithm is designed such that it can perform parameter free clustering on a dataset, where the data vectors are such that there is sufficient difference in the statistical distribution between angles formed by points within a subspace and between subspaces. As observed from Fig. \ref{fig_gamma_K}, whenever this assumption holds, there is a drastic jump in $\gamma_K$ and rapid fall of $\zeta_{K}$ when $K=L$. Thus, the possibility of success of the algorithm, can also be very easily visualized if one were to look at the evolution of $\gamma_K$ and $\zeta_{K}$. In a dataset where the algorithm would do well, the $\gamma_K$ value spikes noticeably at a certain point $K$ where there is also a drastic fall in $\zeta_K$. }
	
	However, if the data are such that the distributions of angles formed by points within a subspace and of those between subspaces are very similar, then the algorithm will fail. There are many datasets, as highlighted in Section \ref{sNumerical_validations}, where the assumption holds approximately, and the algorithm can cluster effectively at high speed, without a hyperparameter. For some other popular datasets like the extended Yale-B, with the feature vector being vectorized image pixel values, $\gamma_K$ evolves smoothly, and $\gamma_{K}$ never crosses $\zeta_{K}$, as shown in Fig. \ref{fig_yaleb_validation} (a), indicating that the algorithm shall fail if we were to use the raw pixels as data vectors. This is potentially due to the fact that the inter-cluster diversity is less in this dataset, as noticed in \cite{zhang2012hybrid}. However, a suitable feature extraction technique, like DSIFT used in Section \ref{sRealDatasets} for extended Yale-B, ensures that there is sufficient separation between the statistical distribution of within-cluster and between cluster angles. This can be seen manifested in the behaviour of $\gamma_K$, as shown in Fig. \ref{fig_yaleb_validation} (b). The difference in the empirical statistical distributions of within-cluster and between-cluster angles before and after feature extraction for extended Yale-B dataset can be seen in Figs. \ref{fig_yaleb_dist} (a) and (b). There could exist a suitable feature extraction technique for every dataset, which could make the algorithm perform effective clustering in that dataset as well. But we have not pursued that line of research here since it is too domain-specific. We have demonstrated in Section \ref{sNumerical_validations}, the effectiveness of the algorithm in diverse domains and not just in image datasets.
	
	{Furthermore, even though our algorithm may fail in some datasets, the possibility of failure can be readily identified by looking at $\gamma_K$ and $\zeta_{K}$, without having to know the ground truth. Since the proposed algorithm is computationally efficient and tuning parameter free, one can apply the method to any dataset and see if the current feature vectors can be clustered effectively with the proposed method, with minimal effort by observing the evolution of $\gamma_K$ and $\zeta_{K}$. Most of the other state of the art methods involve using tuning parameters,  and to set them appropriately, one must have pre-hand knowledge of the ground truth. Even using the ground truth or with training sets, the time it takes to set appropriate tuning parameters is quite high. We provide the run time comparison for a single run of the algorithms in extended Yale-B dataset in Table \ref{tab_r_time}. The second column gives information about the number of parameters required as input to the algorithm. For example, SSC-ADMM requires the number of clusters and one more tuning parameter. All the algorithms are run on the same system for fairness, and all algorithms are provided with the true number of clusters as an input. Note that our algorithm is roughly $8$ times faster than SSC-ADMM, $550$ times faster than ALC and $2500$ times faster than LRR. It is evident that many existing algorithms take more time to obtain clustering, even with the predefined set of parameters. Setting the tuning parameters would take significant multiple of this time unless there is a predefined way to tune parameters other than grid search and cross-validation. }
	\section{Conclusions}
	In this paper, we have proposed a parameter free algorithm for subspace clustering, which distinguishes between the points from different subspaces using the characteristics of the distribution of angles subtended by the points. The algorithm, which works without apriori parameter knowledge, starts with a fine clustering and merges the clusters iteratively until the clustering score crosses a threshold. We have theoretically analysed the algorithm and derived the threshold under an assumption on the data model and also proposed a parameter free initial clustering method. The performance of the proposed algorithm has been studied {extensively in both synthetic as well as many real datasets}. It has been observed that the proposed method performs on par with other existing methods which use true parameter knowledge, in terms of clustering error and estimated number of clusters and outperforms them in many cases, especially when the true parameters are unknown. {In this work, we have used empirical Bhattacharyya distance as a discriminating criterion. However, one can use any other statistical distance provided one can derive the appropriate thresholds. This could be an interesting direction for future research.}

	\appendices
	\section{Proofs of results in Section \ref{sassuumpttions}}\label{appendix1}
	\begin{proof}[Proof of Lemma \ref{lemmaold}]
		Results in \cite{menon2019structured} is built on the basis of Lemma 12 from \cite{cai2013distributions} which gives the distribution of angles between randomly chosen points in $\mathbb{S}^{n-1}$. 
		Part a) is straight from Lemma 2 in \cite{menon2019structured}. Note that the angle between two points from different subspaces in Model  \ref{model_ss} is statistically same as that between two points chosen uniformly at random from $\mathbb{S}^{n-1}$ as in Lemma 9 in \cite{menon2019structured} and part b) follows Lemma 1  in \cite{menon2019structured}. For the convergence in distribution in (b):
				Let $\tau=\sqrt{n-2}\left(\theta_{ij}-\frac{\pi}{2}\right)$. Using  
		expression for $h_{n}(\theta)$, the log density of $\tau$ can be obtained as
		\begin{equation*}
		\begin{split}	\log_e g(\tau) =C_n+(n-2)&\log_e\cos\left(\frac{\tau}{\sqrt{n-2}}\right),\\ &\tau\in\left[-\sqrt{n-2}\,\frac{\pi}{2} ,	 \sqrt{n-2}\,\frac{\pi}{2}\right],
		\end{split}
		\end{equation*}
		where $\exp({C_n})$ is the normalization term depending on $n$ alone. Using Taylor expansion about $\tau=0$,
		\begin{equation*}
		\begin{split}
		\log_eg(\tau)&=C_n+ (n-2)\left[-\frac{\tau^2}{2(n-2)}-\frac{\tau^4}{12(n-2)^2}-\ldots\right] \\ 
		&=C_n-\frac{\tau^2}{2}+{O}\left(\frac{1}{n}\right).
		\end{split}
		\end{equation*}
		$\Rightarrow g(\tau)\propto e^{-\frac{\tau^2}{2}}\mbox{ at the rate of }{O}\left(\frac{1}{n}\right)$. Thus,  $\tau\xrightarrow{\mathcal{D}}\mathcal{N}(0,1)$ and hence $\theta_{ij}\xrightarrow{\mathcal{D}}\mathcal{N}\left(\frac{\pi}{2},\frac{1}{n-2}\right)$ at the rate of ${O}\left(\frac{1}{n}\right)$.
	\end{proof}
	\section{Proof of results in Section \ref{stheory}}\label{appendix2}
	The following lemma is used to design subsets of independent samples for calculation of estimates.
	\begin{lemma}\label{lsubsets}
		Let the constituent clusters be $I_i = \{i_1,i_2,\ldots, i_{\omega_i}\}$ and $I_j = \{j_1,j_2,\ldots ,j_{\omega_j}\}$ and let them contain only indices of points from the same subspace $\mathcal{U}_a$. Then, $W_i = \{\theta_{i_pi_q}\text{	}|\text{	} p,q\!=\!1,2,\ldots, \omega_i, \ p\! <\! q  \}$ and $B_{ij} = \{\theta_{i_pj_q}\text{	}|\text{	} p = 1,2,\ldots, \omega_i, \  q = 1,2, \ldots ,\omega_j  \}$. Define $W^{t_i}_i = \{\theta_{i_{(2k-1)}i_{(2k)}}\text{	}|\text{	}k = 1,2,\ldots, \lfloor\frac{\omega_i}{2}\rfloor\}$. Then, $|W^{t_i}_i| = t_i = \lfloor \frac{\omega_i}{2} \rfloor$. Let $\omega = \min (\omega_i,\omega_j)$. Define $B^{t_j}_{ij} = \{\theta_{i_pj_p}\text{	}|\text{	} p = 1,2,\ldots, \omega \}$. Then, $|B^{t_j}_{ij}| = t_j = \omega$. The following holds on the estimates under Assumption \ref{amain1}:
		\begin{itemize}
			\item[a)] $\mu_{w_{i}t_i}$ and $\mu_{b_{ij}t_j}$ are calculated using independent angle samples. 
			\item[b)] $\mu_{w_{i}t_i}$ and $\mu_{b_{ij}t_j}$ are independent.
			\item[c)] The corresponding variance estimates $\sigma^2_{w_{i}t_i}$ and $\sigma^2_{b_{ij}t_j}$ are also independent.
		\end{itemize}
	\end{lemma}
	\begin{proof}
		As seen in the design of the set $W^{t_i}_i$, only one angle is chosen per data point in the set and hence the estimates using this set uses independent samples under Assumption \ref{amain1}. Also, the between angle set $B^{t_j}_{ij}$  contains only one angle per data point which are independent. This proves part a). When we see both the sets together, they contain at most 2 angles formed by a point and under Assumption \ref{amain1}, the angles are pairwise independent if it involves the same point. Hence, $W^{t_i}_i \cup B^{t_j}_{ij}$ contains only independent samples. Hence, the estimates which use disjoint samples from $W^{t_i}_i \cup B^{t_j}_{ij}$ are independent. This proves part b) and c). 
	\end{proof}
	{	
	\begin{proof}[Proof of Lemma \ref{ldistr}]
		Part a): Since we are working with independent samples as designed in Lemma \ref{lsubsets} sampled from a Gaussian distribution with mean $\nu_a$ and variance $\rho_a^2$ as per Assumption \ref{amain1}, $\mu_{w_{i}t_{ij}}, \mu_{b_{ij}t_{ij}} \sim \mathcal{N}(\nu_{a},\frac{\rho_a^2}{t_{ij}})$ and $\sigma^2_{w_{i}t_{ij}},\ \sigma^2_{b_{ij}t_{ij}} \sim \frac{\rho_a^2}{t_{ij}-1} \chi^2_{t_{ij}-1}$
		which is a straight forward result on sample mean and variance of a Gaussian random sample as in Theorem 5.3.1 in \cite{casella2002statistical}. Then, $X_{ij}$ is the difference of independent Gaussian random variables with same mean. 
		\begin{align*}
			\Rightarrow \quad X_{ij}  &\sim \mathcal{N}\left(0,\frac{2\rho_{a}^2}{t}\right) \ \mbox{ and }\
			\frac{t}{2\rho_{a}^2}\,X_{ij} ^2 \sim  \chi^2_{1}. 
		\end{align*}
		Part b) Similar to Part a), here we have $$\sigma^2_{w_{i}t_{ij}} \sim \frac{\rho_a^2}{t_{ij}-1} \chi^2_{t_{ij}-1}\qquad \sigma^2_{b_{ij}t_{ij}} \sim \frac{\rho_{ab}^2}{t_{ij}-1} \chi^2_{t_{ij}-1}$$ $$(\mu_{w_{i}t_{ij}} - \mu_{b_{ij}t_{ij}}) \sim \mathcal{N}\left(\nu_a - \nu_{ab}, \frac{\rho_a^2+\rho_{ab}^2}{t_{ij}}\right).$$
		From the distribution of difference in means, we can also define the distribution of its square as a scaled non-central $\chi^2$ distribution.
		\begin{align*}
			\frac{t_{ij}}{\rho_a^2+\rho_{ab}^2}X_{ij}^2 \sim \chi^2\left(k = 1, \lambda = t_{ij}\frac{(\nu_a-\nu_{ab})^2}{\rho_a^2+\rho_{ab}^2}\right). \tag*{\qedhere}
		\end{align*}
	\end{proof}}
{\begin{proof}[Proof of Theorem \ref{t2}]
	From (\ref{edistd}) in Lemma \ref{ldistr}:
	\begin{equation}\label{esammu}
	\begin{aligned}
	&\frac{t_{ij}}{\rho_a^2+\rho_{ab}^2}X_{ij}^2 \geq  t_{ij} \log_{e}(1+M_{ab}) \alpha_{t_{ij}}\hskip5pt w.p\\&\qquad\qquad 
	1 -F_{\chi^2(1,t_{ij}M_{ab}^2)}\big(t_{ij}\log_{e}(1+M_{ab}) \alpha_{t_{ij}}\big)
	\end{aligned}
	\end{equation}
	Using the distributions for sample variances in (\ref{edistd}), 
	\begin{equation}\label{esamvar}
	\begin{aligned}
	&2\rho^2_a - \rho^2_a \alpha_{t_{ij}} \leq \sigma^2_{w_{i}t_{ij}} \leq \rho^2_a\alpha_{t_{ij}}\qquad w.p\\&
	F_{\chi^2_{t_{ij}-1}}\big((t_{ij}-1)\alpha_{t_{ij}}\big) - F_{\chi^2_{t_{ij}-1}}\big((t_{ij}-1)(2-\alpha_{t_{ij}})\big)
	\end{aligned}
	\end{equation}
	and
	\begin{equation}\label{esamvar1}
	\begin{aligned}
	&2\rho^2_{ab} - \rho^2_{ab} \alpha_{t_{ij}} \leq \sigma^2_{b_{ij}t_{ij}} \leq \rho^2_{ab}\alpha_{t_{ij}}\qquad w.p\\& F_{\chi^2_{t_{ij}-1}}\big((t_{ij}-1)\alpha_{t_{ij}}\big) - F_{\chi^2_{t_{ij}-1}}\big((t_{ij}-1)(2-\alpha_{t_{ij}})\big).
	\end{aligned}
	\end{equation}
	The independent events (\ref{esammu}), (\ref{esamvar}) and (\ref{esamvar1}) %
	occur together $w.p$ 
	\begin{align*}
		&\Big[F_{\chi^2_{t_{ij}-1}}\big((t_{ij}-1)\alpha_{t_{ij}}\big) - F_{\chi^2_{t_{ij}-1}}\big((t_{ij}-1)(2-\alpha_{t_{ij}})\big)\Big]^2\!\times\!\Big[1 \\&-F_{\chi^2(1,t_{ij}M_{ab}^2)}\big(t_{ij}\log_{e}(1+M_{ab}) \alpha_{t_{ij}}\big)\Big] = 1 - \delta_{t_{ij}}^{ab}.\mbox{
		Consider }
	\end{align*}
	$Y_{ij}\leq (\rho_a^2+\rho_{ab}^2)\alpha_{t_{ij}}$. 
	From (\ref{esamvar}) and (\ref{esamvar1}), This  occurs $w.p\ \geq 1 - \delta_{t_{ij}}^{ab}$. Thus, $U_{ij}  = {X_{ij}^2}/{Y_{ij}}
	\geq \dfrac{X_{ij}^2}{(\rho_a^2+\rho_{ab}^2)\alpha_{t_{ij}}}
	$
	. Using (\ref{esammu}),
	\begin{align} \label{efirstterm}
	\frac{1}{4}U_{ij} \geq \frac{\log_{e}(1+M_{ab})}{4} \qquad w.p\ \geq 1 - \delta_{t_{ij}}^{ab}.
	\end{align}
	Also from (\ref{esamvar}) and (\ref{esamvar1}),
	\begin{align*}
	V_{ij} = \frac{\sigma^2_{w_{i}t_{ij}}}{\sigma^2_{b_{ij}t_{ij}}} + \frac{\sigma^2_{b_{ij}t_{ij}}}{\sigma^2_{w_{i}t_{ij}}} &\geq \frac{\rho_a^2(2-\alpha_{t_{ij}})}{\rho_{ab}^2\alpha_{t_{ij}}} + \frac{\rho_{ab}^2(2-\alpha_{t_{ij}})}{\rho_a^2\alpha_{t_{ij}}}\\
	&=\frac{2-\alpha_{t_{ij}}}{\alpha_{t_{ij}}} R_{ab}.
	\end{align*}
	Hence, $w.p\ \geq 1 - \delta_{t_{ij}}^{ab},$
	\begin{equation}\label{esecondterm}
	\frac{1}{4}\log_e\!\Big[\frac{ V_{ij}}{4} +\frac{1}{2}\Big] \geq \frac{1}{4}\log_e\!\Big[ \Big(\frac{2-\alpha_{t_{ij}}}{\alpha_{t_{ij}}}\Big)\frac{R_{ab}}{4} +\frac{1}{2}\Big].
	\end{equation}
	Now let us look at $d_{ij}$. From (\ref{efirstterm}) and (\ref{esecondterm}),
	\begin{equation*}
	d_{ij}  \geq \frac{\log_{e}(1+M_{ab})}{4} + \frac{1}{4}\log_e\Big[ \Big(\frac{2-\alpha_{t_{ij}}}{\alpha_{t_{ij}}}\Big)\frac{R_{ab}}{4} +\frac{1}{2}\Big].
	\end{equation*} 
	To ensure $d_{ij} \geq \frac{1}{\sqrt{t_{ij}-1}}$ $w.p$ $\geq 1-\delta_{t_{ij}}^{ab}$, it is sufficient that:
	$$\frac{\log_{e}(1+M_{ab})}{4} + \frac{1}{4}\log_e \Big[ \Big(\frac{2-\alpha_{t_{ij}}}{\alpha_{t_{ij}}}\Big)\frac{R_{ab}}{4} +\frac{1}{2}\Big] \geq \frac{1}{\sqrt{t_{ij}-1}}.$$
	$$\Rightarrow \Big(\frac{2-\alpha_{t_{ij}}}{\alpha_{t_{ij}}}\Big)R_{ab}(1+M_{ab}) \geq  4\alpha_{t_{ij}} - 2(1+M_{ab}).$$
	\begin{equation}\label{equad}
	\Rightarrow 4\alpha_{t_{ij}}^2 + (R_{ab}-2)(1+M_{ab})\alpha_{t_{ij}} - 2R_{ab}(1+M_{ab}) \leq 0.
	\end{equation}
	The roots of the above quadratic are obtained as $(\psi_{ab},\psi'_{ab})=$
	\begin{align*}
		\tfrac{{-(1+M_{ab})(R_{ab}-2) \pm \sqrt{(1+M_{ab})^2(R_{ab}-2)^2+32R_{ab}(1+M_{ab})}}}{8}. \mbox{ If } x>0
		\end{align*}
	then $x+x^{-1}\geq 2\Rightarrow 
	R_{ab}\geq 2$. Also, $M_{ab} \geq 0$. Thus, the roots are real and $\psi'_{ab}\leq\psi_{ab}$. Therefore, from (\ref{equad}), we have  $4(\alpha_{t_{ij}}-\psi'_{ab})(\alpha_{t_{ij}}-\psi_{ab})\leq0\Rightarrow\psi'_{ab}\leq\alpha_{t_{ij}}\leq\psi_{ab}$. If we assume $\psi_{ab}<1$ and use the fact that $M_{ab}\geq0$, we will arrive at the contradiction $R_{ab}<2$. Thus, $\psi_{ab}\geq1$. Also note that $\alpha_{t_{ij}} \geq 1$ and $\psi'_{ab}\leq0$. Hence, $1\leq\alpha_{t_{ij}}\leq\psi_{ab}$ where
	\begin{align*}
	\psi_{ab} = \tfrac{\sqrt{(R_{ab}-2)^2(1+M_{ab})^2 +32R_{ab}(1+M_{ab})} -(R_{ab}-2)(1+M_{ab}) }{8}.
	\end{align*}
	Thus, $d_{ij} \geq \frac{1}{\sqrt{t_{ij}-1}}$, $w.p$ $\geq 1-\delta_{t_{ij}}^{ab}$ if $\alpha_{t_{ij}} \leq \psi_{ab}\ $ \begin{align*} \Rightarrow \ t_{ij} \geq 1+ \frac{16}{(\log_e\psi_{ab})^2}.\tag*{\qedhere}\end{align*}
\end{proof}}
	{
		\bibliographystyle{IEEEtran}
		\bibliography{bibl_final_draft}

\begin{thebibliography}{10}
\providecommand{\url}[1]{#1}
\csname url@samestyle\endcsname
\providecommand{\newblock}{\relax}
\providecommand{\bibinfo}[2]{#2}
\providecommand{\BIBentrySTDinterwordspacing}{\spaceskip=0pt\relax}
\providecommand{\BIBentryALTinterwordstretchfactor}{4}
\providecommand{\BIBentryALTinterwordspacing}{\spaceskip=\fontdimen2\font plus
\BIBentryALTinterwordstretchfactor\fontdimen3\font minus
  \fontdimen4\font\relax}
\providecommand{\BIBforeignlanguage}[2]{{%
\expandafter\ifx\csname l@#1\endcsname\relax
\typeout{** WARNING: IEEEtran.bst: No hyphenation pattern has been}%
\typeout{** loaded for the language `#1'. Using the pattern for}%
\typeout{** the default language instead.}%
\else
\language=\csname l@#1\endcsname
\fi
#2}}
\providecommand{\BIBdecl}{\relax}
\BIBdecl

\bibitem{xu2005survey}
R.~Xu and D.~C. Wunsch~II, ``Survey of clustering algorithms,'' \emph{IEEE
  Trans. Neural Networks}, vol.~16, no.~3, pp. 645--678, 2005.

\bibitem{beyer1999nearest}
K.~Beyer, J.~Goldstein, R.~Ramakrishnan, and U.~Shaft, ``When is “nearest
  neighbor” meaningful?'' in \emph{Int. Conf. Database Theory}.\hskip 1em
  plus 0.5em minus 0.4em\relax Springer, 1999, pp. 217--235.

\bibitem{parsons2004subspace}
L.~Parsons, E.~Haque, and H.~Liu, ``Subspace clustering for high dimensional
  data: a review,'' \emph{ACM SIGKDD Explorations Newsletter}, vol.~6, no.~1,
  pp. 90--105, 2004.

\bibitem{Cherkassky1998Learning}
V.~S. Cherkassky and F.~Mulier, \emph{Learning from Data: Concepts, Theory, and
  Methods}, 1st~ed.\hskip 1em plus 0.5em minus 0.4em\relax New York, NY, USA:
  John Wiley \& Sons, Inc., 1998.

\bibitem{basri2003lambertian}
R.~Basri and D.~W. Jacobs, ``Lambertian reflectance and linear subspaces,''
  \emph{IEEE Trans. Pattern Analysis and Machine Intelligence}, vol.~25, no.~2,
  pp. 218--233, 2003.

\bibitem{jolliffe2002principal}
I.~Jolliffe, \emph{Principal component analysis (Springer Series in
  Statistics)}.\hskip 1em plus 0.5em minus 0.4em\relax Berlin, Germany:
  Springer, 2002.

\bibitem{vidal2011subspace}
R.~Vidal, ``Subspace clustering,'' \emph{IEEE Signal Processing Magazine},
  vol.~28, no.~2, pp. 52--68, 2011.

\bibitem{hong2006multiscale}
W.~Hong, J.~Wright, K.~Huang, and Y.~Ma, ``Multiscale hybrid linear models for
  lossy image representation,'' \emph{IEEE Trans. Image Processing}, vol.~15,
  no.~12, pp. 3655--3671, 2006.

\bibitem{vidal2008multiframe}
R.~Vidal, R.~Tron, and R.~Hartley, ``Multiframe motion segmentation with
  missing data using powerfactorization and gpca,'' \emph{Int. J. Computer
  Vision}, vol.~79, no.~1, pp. 85--105, 2008.

\bibitem{ho2003clustering}
J.~Ho, M.-H. Yang, J.~Lim, K.-C. Lee, and D.~Kriegman, ``Clustering appearances
  of objects under varying illumination conditions,'' in \emph{IEEE Conf.
  Computer Vision and Pattern Recognition}, vol.~1.\hskip 1em plus 0.5em minus
  0.4em\relax IEEE, 2003, pp. 11--18.

\bibitem{yang2008unsupervised}
A.~Y. Yang, J.~Wright, Y.~Ma, and S.~S. Sastry, ``Unsupervised segmentation of
  natural images via lossy data compression,'' \emph{Computer Vision and Image
  Understanding}, vol. 110, no.~2, pp. 212--225, 2008.

\bibitem{vidal2005generalized}
R.~Vidal, Y.~Ma, and S.~Sastry, ``Generalized principal component analysis
  (gpca),'' \emph{IEEE Trans. Pattern Analysis and Machine Intelligence},
  vol.~27, no.~12, pp. 1945--1959, 2005.

\bibitem{vidal2003algebraic}
R.~Vidal, S.~Soatto, Y.~Ma, and S.~Sastry, ``An algebraic geometric approach to
  the identification of a class of linear hybrid systems,'' in \emph{Int. Conf.
  on Decision and Control}.\hskip 1em plus 0.5em minus 0.4em\relax IEEE, 2003,
  pp. 167--172.

\bibitem{jiang2004cluster}
D.~Jiang, C.~Tang, and A.~Zhang, ``Cluster analysis for gene expression data: a
  survey,'' \emph{IEEE Trans. Knowledge and Data Engineering}, vol.~16, no.~11,
  pp. 1370--1386, 2004.

\bibitem{achtert2006finding}
E.~Achtert, C.~B{\"o}hm, H.-P. Kriegel, P.~Kr{\"o}ger, I.~M{\"u}ller-Gorman,
  and A.~Zimek, ``Finding hierarchies of subspace clusters,'' in \emph{European
  Conf. Principles of Data Mining and Knowledge Discovery}.\hskip 1em plus
  0.5em minus 0.4em\relax Springer, 2006, pp. 446--453.

\bibitem{agarwal2005research}
N.~Agarwal, E.~Haque, H.~Liu, and L.~Parsons, ``Research paper recommender
  systems: A subspace clustering approach,'' in \emph{Int. Conf. Web-Age Info.
  Management}.\hskip 1em plus 0.5em minus 0.4em\relax Springer, 2005, pp.
  475--491.

\bibitem{zhou2014text}
X.~Zhou, J.~Liang, Y.~Hu, and L.~Guo, ``Text document latent subspace
  clustering by plsa factors,'' in \emph{IEEE/WIC/ACM Int. Joint Conf. Web
  Intelligence (WI) and Intelligent Agent Technologies (IAT)}, vol.~2.\hskip
  1em plus 0.5em minus 0.4em\relax IEEE, 2014, pp. 442--448.

\bibitem{wu2001multibody}
Y.~Wu, Z.~Zhang, T.~S. Huang, and J.~Y. Lin, ``Multibody grouping via
  orthogonal subspace decomposition,'' in \emph{IEEE Conf. Computer Vision and
  Pattern Recognition}, vol.~2.\hskip 1em plus 0.5em minus 0.4em\relax IEEE,
  2001, pp. 252--257.

\bibitem{zhang2009median}
T.~Zhang, A.~Szlam, and G.~Lerman, ``Median k-flats for hybrid linear modeling
  with many outliers,'' in \emph{Int. Conf. Computer Vision Workshops}.\hskip
  1em plus 0.5em minus 0.4em\relax IEEE, 2009, pp. 234--241.

\bibitem{ma2007segmentation}
Y.~Ma, H.~Derksen, W.~Hong, and J.~Wright, ``Segmentation of multivariate mixed
  data via lossy data coding and compression,'' \emph{IEEE Trans. Pattern
  Analysis and Machine Intelligence}, vol.~29, no.~9, pp. 1546--1562, 2007.

\bibitem{von2007tutorial}
U.~Von~Luxburg, ``A tutorial on spectral clustering,'' \emph{Statistics and
  Computing}, vol.~17, no.~4, pp. 395--416, 2007.

\bibitem{elhamifar2013sparse}
E.~Elhamifar and R.~Vidal, ``Sparse subspace clustering: Algorithm, theory, and
  applications,'' \emph{IEEE Trans. Pattern Analysis and Machine Intelligence},
  vol.~35, no.~11, pp. 2765--2781, 2013.

\bibitem{liu2012robust}
G.~Liu, Z.~Lin, S.~Yan, J.~Sun, Y.~Yu, and Y.~Ma, ``Robust recovery of subspace
  structures by low-rank representation,'' \emph{IEEE Trans. Pattern Analysis
  and Machine Intelligence}, vol.~35, no.~1, pp. 171--184, 2012.

\bibitem{dyer2013greedy}
E.~L. Dyer, A.~C. Sankaranarayanan, and R.~G. Baraniuk, ``Greedy feature
  selection for subspace clustering,'' \emph{J. Machine Learning Research},
  vol.~14, no.~1, pp. 2487--2517, 2013.

\bibitem{you2016scalable}
C.~You, D.~Robinson, and R.~Vidal, ``Scalable sparse subspace clustering by
  orthogonal matching pursuit,'' in \emph{IEEE Conf. Computer Vision and
  Pattern Recognition}, 2016, pp. 3918--3927.

\bibitem{lu2012robust}
C.-Y. Lu, H.~Min, Z.-Q. Zhao, L.~Zhu, D.-S. Huang, and S.~Yan, ``Robust and
  efficient subspace segmentation via least squares regression,'' in
  \emph{European Conf. on Computer Vision}.\hskip 1em plus 0.5em minus
  0.4em\relax Springer, 2012, pp. 347--360.

\bibitem{you2016oracle}
C.~You, C.-G. Li, D.~P. Robinson, and R.~Vidal, ``Oracle based active set
  algorithm for scalable elastic net subspace clustering,'' in \emph{IEEE Conf.
  Computer Vision and Pattern Recognition}, 2016, pp. 3928--3937.

\bibitem{favaro2011closed}
P.~Favaro, R.~Vidal, and A.~Ravichandran, ``A closed form solution to robust
  subspace estimation and clustering,'' in \emph{IEEE Conf. Computer Vision and
  Pattern Recognition}.\hskip 1em plus 0.5em minus 0.4em\relax IEEE, 2011, pp.
  1801--1807.

\bibitem{wang2013provable}
Y.-X. Wang, H.~Xu, and C.~Leng, ``Provable subspace clustering: When lrr meets
  ssc,'' in \emph{Adv. Neural Info. Processing Systems}, 2013, pp. 64--72.

\bibitem{wang2019provable}
Y.-X. \hspace{0cm}Wang, H.~Xu, and C.~Leng, ``Provable subspace clustering:
  When lrr meets ssc,'' \emph{IEEE Trans. Info. Theory}, vol.~65, no.~9, pp.
  5406--5432, 2019.

\bibitem{lu2019subspace}
C.~Lu, J.~Feng, Z.~Lin, T.~Mei, and S.~Yan, ``Subspace clustering by block
  diagonal representation.'' \emph{IEEE Trans. Pattern Analysis and Machine
  Intelligence}, vol.~41, no.~2, pp. 487--501, 2019.

\bibitem{leonardis2002multiple}
A.~Leonardis, H.~Bischof, and J.~Maver, ``Multiple eigenspaces,'' \emph{Pattern
  recognition}, vol.~35, no.~11, pp. 2613--2627, 2002.

\bibitem{fan2005multibody}
Z.~Fan, J.~Zhou, and Y.~Wu, ``Multibody grouping by inference of multiple
  subspaces from high-dimensional data using oriented-frames,'' \emph{IEEE
  Trans. Pattern Analysis and Machine Intelligence}, vol.~28, no.~1, pp.
  91--105, 2005.

\bibitem{rahmani2017innovation}
M.~Rahmani and G.~K. Atia, ``Innovation pursuit: A new approach to subspace
  clustering,'' \emph{IEEE Trans. Signal Processing}, vol.~65, no.~23, pp.
  6276--6291, 2017.

\bibitem{min2018survey}
E.~Min, X.~Guo, Q.~Liu, G.~Zhang, J.~Cui, and J.~Long, ``A survey of clustering
  with deep learning: From the perspective of network architecture,''
  \emph{IEEE Access}, vol.~6, pp. 39\,501--39\,514, 2018.

\bibitem{song2013auto}
C.~Song, F.~Liu, Y.~Huang, L.~Wang, and T.~Tan, ``Auto-encoder based data
  clustering,'' in \emph{Iberoamerican Congress on Pattern Recognition}.\hskip
  1em plus 0.5em minus 0.4em\relax Springer, 2013, pp. 117--124.

\bibitem{ji2017deep}
P.~Ji, T.~Zhang, H.~Li, M.~Salzmann, and I.~Reid, ``Deep subspace clustering
  networks,'' in \emph{Adv. Neural Info. Processing Systems}, 2017, pp. 24--33.

\bibitem{chen2018subspace}
Y.~Chen, L.~Zhang, and Z.~Yi, ``Subspace clustering using a low-rank
  constrained autoencoder,'' \emph{Info. Sciences}, vol. 424, pp. 27--38, 2018.

\bibitem{cai2013distributions}
T.~Cai, J.~Fan, and T.~Jiang, ``Distributions of angles in random packing on
  spheres,'' \emph{J. Machine Learning Research}, vol.~14, no.~1, pp.
  1837--1864, 2013.

\bibitem{menon2019structured}
V.~Menon and S.~Kalyani, ``Structured and unstructured outlier identification
  for robust pca: A fast parameter free algorithm,'' \emph{IEEE Trans. Signal
  Processing}, vol.~67, no.~9, pp. 2439--2452, 2019.

\bibitem{heckel2015robust}
R.~Heckel and H.~B{\"o}lcskei, ``Robust subspace clustering via thresholding,''
  \emph{IEEE Trans. Info. Theory}, vol.~61, no.~11, pp. 6320--6342, 2015.

\bibitem{rahmani2017coherence}
M.~Rahmani and G.~K. Atia, ``Coherence pursuit: Fast, simple, and robust
  principal component analysis,'' \emph{IEEE Trans. Signal Processing},
  vol.~65, no.~23, pp. 6260--6275, 2017.

\bibitem{gitlin2018improving}
A.~Gitlin, B.~Tao, L.~Balzano, and J.~Lipor, ``Improving $ k $-subspaces via
  coherence pursuit,'' \emph{IEEE J. Selected Topics in Signal Processing},
  vol.~12, no.~6, pp. 1575--1588, 2018.

\bibitem{lipor2017subspace}
J.~Lipor, D.~Hong, Y.~S. Tan, and L.~Balzano, ``Subspace clustering using
  ensembles of $k$-subspaces,'' \emph{arXiv preprint arXiv:1709.04744}, 2017.

\bibitem{dopazo2001methods}
J.~Dopazo, E.~Zanders, I.~Dragoni, G.~Amphlett, and F.~Falciani, ``Methods and
  approaches in the analysis of gene expression data,'' \emph{J. immunological
  methods}, vol. 250, no. 1-2, pp. 93--112, 2001.

\bibitem{comiter2016lambda}
M.~Comiter, M.~Cha, H.~Kung, and S.~Teerapittayanon, ``Lambda means clustering:
  automatic parameter search and distributed computing implementation,'' in
  \emph{Int. Conf. Pattern Recognition}.\hskip 1em plus 0.5em minus 0.4em\relax
  IEEE, 2016, pp. 2331--2337.

\bibitem{kulis2011revisiting}
B.~Kulis and M.~I. Jordan, ``Revisiting k-means: New algorithms via bayesian
  nonparametrics,'' in \emph{Int. Conf. Machine Learning}.\hskip 1em plus 0.5em
  minus 0.4em\relax Omnipress, 2012, pp. 1131--1138.

\bibitem{claesen2015hyperparameter}
M.~Claesen and B.~De~Moor, ``Hyperparameter search in machine learning,''
  \emph{arXiv preprint arXiv:1502.02127}, 2015.

\bibitem{pelleg2000x}
D.~Pelleg and A.~W. Moore, ``X-means: Extending k-means with efficient
  estimation of the number of clusters,'' in \emph{Int. Conf. Machine
  Learning}.\hskip 1em plus 0.5em minus 0.4em\relax Morgan Kaufmann Publishers
  Inc., 2000, pp. 727--734.

\bibitem{tibshirani2001estimating}
R.~Tibshirani, G.~Walther, and T.~Hastie, ``Estimating the number of clusters
  in a data set via the gap statistic,'' \emph{J. the Royal Statistical
  Society: Series B}, vol.~63, no.~2, pp. 411--423, 2001.

\bibitem{salvador2004determining}
S.~Salvador and P.~Chan, ``Determining the number of clusters/segments in
  hierarchical clustering/segmentation algorithms,'' in \emph{Int. Conf. Tools
  with Artificial Intelligence}.\hskip 1em plus 0.5em minus 0.4em\relax IEEE,
  2004, pp. 576--584.

\bibitem{gupta2018fast}
A.~Gupta, S.~Datta, and S.~Das, ``Fast automatic estimation of the number of
  clusters from the minimum inter-center distance for k-means clustering,''
  \emph{Pattern Recognition Letters}, vol. 116, pp. 72--79, 2018.

\bibitem{DBLP:conf/icml/KallummilK18}
S.~\hspace{0mm}Kallummil and S.~Kalyani, ``Signal and noise statistics
  oblivious orthogonal matching pursuit,'' in \emph{Int. Conf. Machine
  Learning}, 2018, pp. 2434--2443.

\bibitem{kallummil2018high}
S.~\vspace{0mm}Kallummil and S.~Kalyani, ``High snr consistent compressive
  sensing without signal and noise statistics,'' \emph{arXiv preprint
  arXiv:1811.07131}, 2018.

\bibitem{kallummil2018noise}
S.~Kallummil and S.~Kalyani, ``Noise statistics oblivious gard for robust
  regression with sparse outliers,'' \emph{IEEE Trans. Signal Processing},
  vol.~67, no.~2, pp. 383--398, 2019.

\bibitem{liu2010robust}
G.~Liu, Z.~Lin, and Y.~Yu, ``Robust subspace segmentation by low-rank
  representation.'' in \emph{Int. Conf. Machine learning}, vol.~1, 2010, p.~8.

\bibitem{weinstein2013cancer}
J.~N. Weinstein, E.~A. Collisson, G.~B. Mills, K.~R.~M. Shaw, B.~A. Ozenberger,
  K.~Ellrott, I.~Shmulevich, C.~Sander, J.~M. Stuart, C.~G. A.~R. Network
  \emph{et~al.}, ``The cancer genome atlas pan-cancer analysis project,''
  \emph{Nature genetics}, vol.~45, no.~10, p. 1113, 2013.

\bibitem{Dua:2019}
\BIBentryALTinterwordspacing
D.~Dua and C.~Graff, ``{UCI} machine learning repository,'' 2017. [Online].
  Available: \url{http://archive.ics.uci.edu/ml}
\BIBentrySTDinterwordspacing

\bibitem{novartis}
\BIBentryALTinterwordspacing
``Cancer program datasets.'' [Online]. Available:
  \url{http://portals.broadinstitute.org/cgi-bin/cancer/datasets.cgi}
\BIBentrySTDinterwordspacing

\bibitem{rohra2017user}
J.~G. Rohra, B.~Perumal, S.~J. Narayanan, P.~Thakur, and R.~B. Bhatt, ``User
  localization in an indoor environment using fuzzy hybrid of particle swarm
  optimization \& gravitational search algorithm with neural networks,'' in
  \emph{Int. Conf. Soft Computing for Problem Solving}.\hskip 1em plus 0.5em
  minus 0.4em\relax Springer, 2017, pp. 286--295.

\bibitem{hastie1995penalized}
T.~Hastie, A.~Buja, and R.~Tibshirani, ``Penalized discriminant analysis,''
  \emph{The Annals of Statistics}, pp. 73--102, 1995.

\bibitem{lecun1998gradient}
Y.~LeCun, L.~Bottou, Y.~Bengio, and P.~Haffner, ``Gradient-based learning
  applied to document recognition,'' \emph{Proc. of the IEEE}, vol.~86, no.~11,
  pp. 2278--2324, 1998.

\bibitem{GeBeKr01}
A.~Georghiades, P.~Belhumeur, and D.~Kriegman, ``From few to many: Illumination
  cone models for face recognition under variable lighting and pose,''
  \emph{IEEE Trans. Pattern Analysis and Machine Intelligence}, vol.~23, no.~6,
  pp. 643--660, 2001.

\bibitem{KCLee05}
K.~Lee, J.~Ho, and D.~Kriegman, ``Acquiring linear subspaces for face
  recognition under variable lighting,'' \emph{IEEE Trans. Pattern Analysis and
  Machine Intelligence}, vol.~27, no.~5, pp. 684--698, 2005.

\bibitem{tron2007benchmark}
R.~Tron and R.~Vidal, ``A benchmark for the comparison of 3-d motion
  segmentation algorithms,'' in \emph{IEEE Conf. Computer Vision and Pattern
  Recognition}.\hskip 1em plus 0.5em minus 0.4em\relax IEEE, 2007, pp. 1--8.

\bibitem{bhattacharyya1943measure}
A.~Bhattacharyya, ``On a measure of divergence between two statistical
  populations defined by their probability distributions,'' \emph{Bulletin of
  the Calcutta Mathematical Society}, vol.~35, pp. 99--109, 1943.

\bibitem{4309450}
A.~K. Jain, ``On an estimate of the bhattacharyya distance,'' \emph{IEEE Trans.
  Systems, Man, and Cybernetics}, vol.~6, no.~11, pp. 763--766, 1976.

\bibitem{park2014greedy}
D.~Park, C.~Caramanis, and S.~Sanghavi, ``Greedy subspace clustering,'' in
  \emph{Adv. Neural Info. Processing Systems}, 2014, pp. 2753--2761.

\bibitem{cai2012phase}
T.~T. Cai and T.~Jiang, ``Phase transition in limiting distributions of
  coherence of high-dimensional random matrices,'' \emph{J. Multivariate
  Analysis}, vol. 107, pp. 24--39, 2012.

\bibitem{soltanolkotabi2012geometric}
M.~Soltanolkotabi, E.~J. Candes \emph{et~al.}, ``A geometric analysis of
  subspace clustering with outliers,'' \emph{The Annals of Statistics},
  vol.~40, no.~4, pp. 2195--2238, 2012.

\bibitem{johnson1995continuous}
N.~L. Johnson, S.~Kotz, and N.~Balakrishnan, \emph{Continuous Univariate
  Distributions}.\hskip 1em plus 0.5em minus 0.4em\relax John Wiley and Sons,
  New York, 1995, vol.~2.

\bibitem{strehl2002cluster}
A.~Strehl and J.~Ghosh, ``Cluster ensembles---a knowledge reuse framework for
  combining multiple partitions,'' \emph{J. Machine Learning Research}, vol.~3,
  no. Dec, pp. 583--617, 2002.

\bibitem{bruna2013invariant}
J.~Bruna and S.~Mallat, ``Invariant scattering convolution networks,''
  \emph{IEEE Trans. Pattern Analysis and Machine Intelligence}, vol.~35, no.~8,
  pp. 1872--1886, 2013.

\bibitem{peng2018structured}
X.~Peng, J.~Feng, S.~Xiao, W.-Y. Yau, J.~T. Zhou, and S.~Yang, ``Structured
  autoencoders for subspace clustering,'' \emph{IEEE Trans. Image Processing},
  vol.~27, no.~10, pp. 5076--5086, 2018.

\bibitem{zhang2012hybrid}
T.~Zhang, A.~Szlam, Y.~Wang, and G.~Lerman, ``Hybrid linear modeling via local
  best-fit flats,'' \emph{Int. J. Computer Vision}, vol. 100, no.~3, pp.
  217--240, 2012.

\bibitem{casella2002statistical}
G.~Casella and R.~L. Berger, \emph{Statistical inference}.\hskip 1em plus 0.5em
  minus 0.4em\relax Duxbury Pacific Grove, CA, 2002, vol.~2.

\end{thebibliography}
	}
	
\end{document}